\documentclass{article}

\PassOptionsToPackage{numbers}{natbib}

\usepackage[preprint]{neurips_2019}

\usepackage[utf8]{inputenc} 
\usepackage[T1]{fontenc}    
\usepackage{hyperref}       
\usepackage{url}            
\usepackage{booktabs}       
\usepackage{amsfonts}       
\usepackage{nicefrac}       
\usepackage{microtype}      
\usepackage{xcolor}

\usepackage{microtype}
\usepackage{mathtools}
\usepackage{graphicx}
\usepackage{enumitem}
\usepackage{booktabs} 

\usepackage{algorithm}
\usepackage{algorithmic}

\usepackage{mathtools}
\usepackage{url}
\usepackage{amsfonts,bm}

\usepackage{booktabs}
\usepackage{amsmath}
\usepackage[T1]{fontenc}
\DeclarePairedDelimiter{\ceil}{\lceil}{\rceil}

\def\R{\mathbb{R}}

\def\omit#1{}
\newcommand\independent{\protect\mathpalette{\protect\independenT}{\perp}}
\def\independenT#1#2{\mathrel{\rlap{$#1#2$}\mkern2mu{#1#2}}}
\usepackage{amsmath,amssymb,amsthm}

\let\emptyset\varnothing

\newtheorem{theorem}{Theorem}

\newtheorem{lemma}{Lemma}[theorem]

\newtheorem{assumption}{Assumption}

\usepackage{lastpage}
\usepackage{newfloat}
\DeclareFloatingEnvironment[name={Figure S}]{suppfigure}
\DeclareFloatingEnvironment[name={Table S}]{supptable}

\def\X{\mathbf{X}}
\newcommand{\rulesep}{\unskip\ \vrule\ }

\usepackage{subcaption}
\usepackage{bbm}

\title{Discovering Nonlinear Relations with Minimum Predictive Information Regularization}

\author{%
  Tailin Wu\thanks{Correspondence to tailin@mit.edu. Major work was done as an intern at NVIDIA.} \\
  MIT \\
  \texttt{tailin@mit.edu} 
  \And
  Thomas Breuel \\
  NVIDIA \\
   \texttt{tbreuel@nvidia.com} \\
   \And
   Michael Skuhersky \\
   MIT \\
   \texttt{vex@mit.edu} \\
    \And
   Jan Kautz \\
   NVIDIA \\
   \texttt{jkautz@nvidia.com} \\
}

\begin{document}

\maketitle

\begin{abstract}
Identifying the underlying directional relations from observational time series with nonlinear interactions and complex relational structures is key to a wide range of applications, yet remains a hard problem. In this work, we introduce a novel minimum predictive information regularization method to infer directional relations from time series, allowing deep learning models to discover nonlinear relations. Our method\footnote{The code for the methods and experiments is open-sourced at \href{https://github.com/tailintalent/causal}{github.com/tailintalent/causal}.} substantially outperforms other methods for learning nonlinear relations in synthetic datasets, and discovers the directional relations in a video game environment and a heart-rate vs. breath-rate dataset. 
\end{abstract}

\section{Introduction and Related Work}

Imagine a dataset with tens or hundreds of observational time series. There may exist interesting directional relations between the time series which we want to uncover, but their relation graph may be complicated, and the relation may be nonlinear as we do not know its functional form. How can we discover the underlying relations of those challenging scenarios in an efficient way, or at least identify candidate relations that are worth further investigation by a researcher? Problems of this type are omnipresent and important in a variety of scientific endeavors and applications, e.g., gene regulatory networks \cite{lozano2009grouped}, neuroscience \cite{neves2008synaptic,seth2015granger}, economics \cite{granger1969investigating,stock1989interpreting} and finance \cite{hiemstra1994testing, granger2000bivariate}.

To address this question, the field of causal learning has proposed a large class of methods to discover or quantify causal relations. These methods have certain limitations in regards to capability of handling nonlinearity, and/or scalability and efficiency to large numbers of time series.
Pearl \cite{pearl2002causality,pearl2009causality,pearl2009causal} defines causality in terms of intervention and structural dependence, under the structural equation models (SEM). However, in our problem, where only observational time series is available, Pearl's definition may not be applicable. In his seminal work, Granger \cite{granger1969investigating,granger1980testing} defines causality via prediction: if the prediction of the future Y via a linear model can be improved by including the information of X, then X causes Y in the Granger sense. The original Granger causality is limited to linear causal models. Although later works also extend Granger causality to kernel methods \cite{ancona2004radial,marinazzo2008kernel, marinazzo2008kernel2, sindhwani2012scalable}, they may still be insufficient to model and discover the nonlinear causal relations in real data. On the other hand, the causal measures of transfer entropy \cite{schreiber2000measuring} and causal influence \cite{janzing2013quantifying} are in theory able to handle any nonlinearity. However, both measures require density estimation of the joint distribution for the full $N$ time series ($N$ is the number of time series), which is difficult and data-hungry when $N$ is large. Constraint-based methods \cite{spirtes2000causation, harris2013pc, pearl2002causality, spirtes2000causation} require repetitive conditional independence tests, where the number of tests will grow large when $N$ is large and the underlying causal graph is dense. Score-based methods search for the structure that yields the optimal score w.r.t. the data, generally using greedy search methods, for example GES \cite{chickering2002optimal}, rankGES \cite{nandy2018high} and GIES \cite{hauser2012characterization}. This in general requires $\Theta(N^2)$ steps, and the number of neighboring states may grow very large at each step. Another closely related field is sparse learning/feature selection methods. Some important classes are Lasso \cite{tibshirani1996regression} and elastic net \cite{zou2005regularization}, which are effective but subject to the limitations of linear models. For nonlinear models, although L1 and group L1 regularization \cite{meier2008group,scardapane2017group,tank2018neural} can induce sparsity in the model parameters, they are model and input dependent. 

To handle the nonlinear relations in time series, a promising tool is neural nets. Not only are neural nets universal function approximators \cite{hornik1991approximation}, a deep neural net also provides exponentially large expressive power \cite{rolnick2018the}, making it particularly suitable for modeling the unknown nonlinear relations in time series. Recently there has been an increasing amount of work on learning the dynamic models of interacting systems \cite{battaglia2016interaction,chang2016compositional,guttenberg2016permutation,watters2017visual,hoshen2017vain,van2018relational}. However, their main focus is to make better predictions, using implicit interaction models (e.g. using fully connected graph networks). In this paper, we are mainly interested in discovering the underlying directional relations in an explicit form, utilizing the expressive power of neural nets.

To discover nonlinear directional relations from potentially large number of time series in an efficient way, the contribution of our work is as follows:
\begin{itemize}[topsep=0pt,partopsep=0pt,leftmargin=*]

\item We introduce a novel relational learning with Minimum Predictive Information Regularization (MPIR) method for exploratory discovery of nonlinear directional relations from observational time series. It is based on minimizing a mutual information-regularized risk with learnable input noise of a prediction model, which allows function approximators such as neural nets to learn nonlinear relations, combining the benefits of the Granger causality paradigm with deep learning models. At the minimization of the objective, the minimum predictive information term quantifies the directional predictive strength between each pair of time series given other time series. For discovering the directional relations among $N$ time series, it only has to learn $N$ models, and does not requires density estimation for the joint $N$ time series.

\item We prove that the minimum predictive information is able to differentiate dependence or independence between pairs of time series, which allows for statistical test. Moreover, we prove that the minimum predictive information is invariant to the scaling of input and reparameterization of the model. We further provide intuition that under certain conditions, our method is likely to discover direct relations instead of indirect associations.

\item We demonstrate on nonlinear synthetic datasets that our method outperforms other methods in discovering true causal relations with larger $N$, and discovers the directional relations in video game environment and real-world heart-rate vs. breath-rate datasets. 

\end{itemize}

\section{Method}

\subsection{Problem setup}
\label{sec:problem_definition}

We consider $N$ time series $x^{(1)}, x^{(2)}, ...x^{(N)}$, where each time series $x^{(i)}=(x^{(i)}_1,x^{(i)}_2,...x^{(i)}_t,...)$ and each $x^{(i)}_t\in \R^M$ is an $M$-dimensional vector. Denote $X_{t-1}^{(i)}=(x^{(i)}_{t-K},x^{(i)}_{t-K+1},...x^{(i)}_{t-1})$ with maximum time horizon of $K$, and $\mathbf{X}_{t-1}=\{X_{t-1}^{(i)}\}, i=1,2,...N$. We also denote $\mathbf{X}_{t-1}^{(\hat{j})}=\mathbf{X}_{t-1}\texttt{\textbackslash} X^{(j)}_{t-1}$ ($\mathbf{X}_{t-1}$ excluding $X^{(j)}_{t-1}$).
We assume that $x^{(1)}, x^{(2)}, ...x^{(N)}$ are generated by stationary response functions $h_i$ that are unknown to the learner:
\begin{equation}
\label{eq:response_function}
    \begin{cases}
      x^{(1)}_t:=h_1(\mathbf{X}_{t-1},u_1) \\
      x^{(2)}_t:=h_2(\mathbf{X}_{t-1},u_2) \\
     ...\\
      x^{(N)}_t:=h_N(\mathbf{X}_{t-1},u_N) 
    \end{cases}
  \end{equation}
for $t=K+1,K+2,...$ . 
Here $u_i\in \R^M, i=1,2,...N$ are noise variables that are mutually independent, are independent of any $X^{(i)}_{t-1}, x^{(i)}_t$, $i\in\{1,2,...N\}$.
For any $i,j\in\{1,2,...N\}$, we assume that the variables $(\mathbf{X}_{t-1}^{(\hat{j})}, X_{t-1}^{(j)}, x_{t}^{(i)})$ have probability density function $P(\mathbf{X}_{t-1}^{(\hat{j})}, X_{t-1}^{(j)}, x_{t}^{(i)})$.

Our method is inspired by Granger causality \cite{granger1969investigating,granger1980testing}, which defines causality via predictions, making it especially suitable for relational inference of observational time series. Adapting to our notation:

\textbf{Granger causality} \cite{granger1980testing}:\textit{
Assuming causal sufficiency \cite{peters2017elements}, we say $X^{(j)}_{t-1}, j\neq i$ does not Granger-cause $x^{(i)}_{t}$, if $P(x^{(i)}_{t}|X^{(j)}_{t-1}, \mathbf{X}_{t-1}^{(\hat{j})})=P(x^{(i)}_{t}|\mathbf{X}_{t-1}^{(\hat{j})})$. Otherwise, we say $X^{(j)}_{t-1}$ Granger-causes $x^{(i)}_{t}$.
}

In practice, we say that time series $j$ Granger-causes time series $i$, if it can be shown via significance tests that the null hypothesis of $P(x^{(i)}_{t}|X^{(j)}_{t-1}, \mathbf{X}_{t-1}^{(\hat{j})})=P(x^{(i)}_{t}|\mathbf{X}_{t-1}^{(\hat{j})})$ is rejected, i.e. $X_{t-1}^{(j)}$ provides statistically significant information for predicting $x_t^{(i)}$.

In his original work, Granger \cite{granger1969investigating} investigates causality with linear function predictors. Later works have extended it to kernel methods \cite{ancona2004radial, marinazzo2008kernel, marinazzo2008kernel2, sindhwani2012scalable}, which essentially estimate linear Granger causality on the feature space of the kernel. To learn potentially highly nonlinear response functions, it may be desirable to use expressive and universal function approximators \cite{hornik1991approximation} such as neural nets. Neural nets are much more flexible than linear models, and do not require kernel selection as in kernel methods.

\subsection{Our method}
\label{sec:our_method}
Based on the definition of Granger causality, a naïve way to combine it with neural net is: for each $j\to i$, train two neural nets, one predicting $x_t^{(i)}$ based on $\X_{t-1}^{(\hat{j})}$, another predicting $x_t^{(i)}$ based on the full $\X_{t-1}=(\X_{t-1}^{(\hat{j})}, X_{t-1}^{(j)})$, and test whether former MSE is significantly larger than the latter. This method suffers from two major drawbacks: (1) instability: different training of the neural net may end up in different local minima, so that the two MSEs have large variance, which is observed in our initial explorations; (2) inefficiency: to discover the relations among $N$ time series, it has to train at least $N^2$ models (for each $x^{(i)}_t$, train $N-1$ models with one $X_{t-1}^{(j)}$ removed, and $Q$ models ($Q\ge1$) with full $\X_{t-1}$ for accumulating statistics). On the other hand, these two drawbacks exactly inspire our method. Instead of predicting $x^{(i)}_t$ with one $X_{t-1}^{(j)}$ missing at a time, what if we let each $X_{t-1}^{(j)}$ have \textit{learnable} corruption, and encourage each $X_{t-1}^{(j)}$ to provide as little information to $x^{(i)}_t$ as possible while maintaining good prediction? In this way, we have a \emph{single} shared model that can span the full product space of $[\text{total corruption}, \text{no corruption}]^{\bigotimes N}$ for $N$ input time series, which is more stable and efficient than the removing one $X_{t-1}^{(j)}$ at a time and training $N$ models. To achieve this, we add independent noise with learnable amplitudes to each input $X_{t-1}^{(j)}$, and measure the corruption by the mutual information between the input and the corrupted input. We then define the following risk:

\begin{equation}
\begin{aligned}
\label{eq:learnable_risk}
R_{\mathbf{X},x^{(i)}}[f_\theta,&\boldsymbol{\eta}]=\mathbb{E}_{\mathbf{X}_{t-1},x_t^{(i)},\boldsymbol{\epsilon}}\left[\left(x_t^{(i)}-f_\theta(\tilde{\mathbf{X}}^{(\boldsymbol{\eta})}_{t-1})\right)^2\right]+\lambda\cdot \sum_{j=1}^{N}I(\tilde{X}^{(j)(\eta_j)}_{t-1};X^{(j)}_{t-1})
\end{aligned}
\end{equation}
where $\tilde{\mathbf{X}}^{(\boldsymbol{\eta})}_{t-1}:=\mathbf{X}_{t-1}+\boldsymbol{\eta}\odot\boldsymbol{\epsilon}$ (or element-wise, $\tilde{X}_{t-1}^{(j)(\eta_j)}:=X_{t-1}^{(j)}+\eta_j\cdot \epsilon_j$, $j=1,2,...N$) is the noise-corrupted inputs with \emph{learnable} noise amplitudes $\eta_j\in \R^{KM}$, and $\epsilon_j\sim N(\mathbf{0},\mathbf{I})$. $\lambda >0$ is a positive hyperparameter for the mutual information $I(\cdot, \cdot)$. 
Intuitively, the minimization of the second term $I(\tilde{X}^{(j)(\eta_j)}_{t-1};X^{(j)}_{t-1})$ requires the noise amplitude $\eta_j$ to go up. The minimization of the first term requires the noise amplitude $\eta_j$ to go down, and the larger causal strength from $X_{t-1}^{(j)}$ to $x_t^{(i)}$, the larger this force. The minimization of the two terms strikes a balance, at which point the $I(\tilde{X}^{(j)(\eta_j)}_{t-1};X^{(j)}_{t-1})$ measures the \emph{minimum} number of bits of information the time series $j$ need to provide to the learner, without compromising the prediction. 
\vskip -0.05in
At the minimization of $R_{\mathbf{X},x^{(i)}}[f_\theta,\boldsymbol{\eta}]$, we define $W_{ji}=I\left(\tilde{X}^{(j)(\eta_j^*)}_{t-1};X^{(j)}_{t-1}\right)$, which we term \emph{minimum predictive information}, where $(f_{\theta^*},\boldsymbol{\eta}^*)=\text{argmin}_{(f_{\theta},\boldsymbol{\eta})}R_{\mathbf{X},x^{(i)}}[f_\theta,\boldsymbol{\eta}]$.
Essentially, $W_{ji}$ measures the \emph{predictive strength} of time series $j$ for predicting time series $i$, \textit{conditioned} on all the other observed time series. We have that
$W_{ji}$ satisfies the following properties: 

\begin{enumerate}[label={(\arabic*)}]
\item If $x^{(j)}\independent x^{(i)}$, then $W_{ji}=0$.
\item $W_{ji}$ is invariant to affine transformation of each individual $X_{t-1}^{(k)},k=1,2,...N$.
\item $W_{ji}$ is invariant to reparameterization of $\theta$ in $f_\theta$ (the mapping remains the same).
\end{enumerate}

The proofs are provided in Appendix \ref{app:W_proof}. Property 1 shows that $W_{ji}$ is able to differentiate time series that are dependent or independent with the target time series $i$. Empirically, to perform statistical tests, we can let the null hypothesis be  $x^{(j)}\independent x^{(i)}$. Before training, we append to $\X_{t-1}$ some fake time series $v^{(s)}_{t-1},s=1,2,...S$ (e.g. by randomly permuting $X_{t-1}^{(j)}$) so that $v^{(s)}_{t-1}\independent x^{(i)}_t$. After optimizing w.r.t. to the augmented dataset, the values of $W_{si}$ between $v^{(s)}_{t-1}$ and $x_t^{(i)}$ form a distribution for which we know that the null hypothesis is true. Then if certain $W_{ji}$ is greater than the $1-\alpha$ quantile (e.g. $\alpha=0.05$) of the distribution, we can reject the null hypothesis of independence. Properties 2 and 3 show the benefit of our method which essentially regularizes the \textit{input information}, compared with L1 and group L1 \cite{meier2008group,scardapane2017group,tank2018neural} which regularize the \textit{model} and thus do not satisfy these two properties.

Moreover, in Appendix \ref{app:W_proof} we further provide intuition that under certain conditions, $W_{ji}$ is likely to favor the time series  that \textit{directly} causes time series $i$, compared with the time series that relate to $i$ via the direct causal connections. 
Note that our method is not \emph{guaranteed} to identify \emph{direct causal} relations (in Granger \cite{granger1980testing} or Pearl \cite{pearl2002causality} sense), which is a very hard problem given the potential large number of time series and nonlinearity present. However, our method provides an effective data exploratory tool to identify time series that are \emph{predictive} of one another, \emph{conditioned} on all the other observed time series, whose identified directional relations can be investigated further by a researcher. As stated above, under certain conditions, our method does favor the direct causal relations. And in the experiment section, we will compare the estimated $W_{ji}$ with true causal relations if available.

\begin{algorithm}[t]
   \caption{\textbf{Relational Learning with Minimum Predictive Information Regularization}}
\label{alg:learnable_noise}
\begin{algorithmic}
   \STATE {\bfseries Require} $x^{(i)}_{t}, \mathbf{X}_{t-1}$, for $i\in\{1,2,...N\},t\in \mathbf{T}=\{K+1,K+2,...\}$.
   \STATE {\bfseries Require $\eta_0$}: a small value for initialization of $\boldsymbol{\eta}$.
   \STATE {\bfseries Require $\lambda$}: coefficient for the mutual information term.
   \STATE {\bfseries Require $S$}: number of fake time series.
   \STATE {\bfseries Require $\alpha$}: significance level.
\STATE 1:\ \ \  Randomly select $S$ indices $i_1,i_2,...i_S$ from $\{1,2,...N\}$
\STATE 2:\ \ \  $v_{t-1}^{(s)}\gets \text{Permute-examples}_t(X_{t-1}^{(i_s)})$ \textbf{for} $s=1,2,...S$    \ \ \ \ // \textit{Permuting on the example dimension}
\STATE 3:\ \ \  $\X^{(\text{aug})}_{t-1}\gets[\X_{t-1}, \mathbf{v}_{t-1}]$, where $\mathbf{v}_{t-1}=[v_{t-1}^{(1)},...v_{t-1}^{(S)}]$ and $[\cdot,...,\cdot]$ denotes concatenation along 
\STATE \ \ \ \ \ \ \ \ \ the dimension of $N$ (thus $\X^{(\text{aug})}_{t-1}$ consists of $N+S$ time series)
\STATE 4:\ \ \ \textbf{for} $i$ in $\{1,2,...N\}\ \textbf{do:}$
\STATE 5:\ \ \ \ \ \ \ \ \ \ Initialize function approximator $f_\theta$.
\STATE 6:\ \ \ \ \ \ \ \ \ \ Initialize $\boldsymbol{\eta}=(\eta_1,\eta_2,...\eta_N)=(\eta_0\boldsymbol{1},\eta_0\boldsymbol{1},...\eta_0\boldsymbol{1})$, where each element $\eta_0\boldsymbol{1}$ is a $KM$-   
 \STATE \ \  \ \ \ \ \ \ \ \ \ \ \ \ \ \ dimensional vector, same dimension as $X_{t-1}^{(j)}$.
\STATE 7:\ \ \ \ \ \ \ \ \ \  $(f_{\theta^*},\boldsymbol{\eta}^*)\gets\text{Minimize}_{(f_\theta,\boldsymbol{\eta})}\hat{R}_{\X^{(\text{aug})},x^{(i)},\boldsymbol{\epsilon}}[f_\theta,\boldsymbol{\eta}]$ (Eq. \ref{eq:empirical}) with e.g. gradient descent.
\STATE 8:\ \ \ \ \ \ \ \ \ \  $W_{ji}\gets I(\tilde{X}^{(j)(\eta_j^*)}_{t-1};X^{(j)}_{t-1})$, for $j=1,2,...N, N+1,...N+S$.
\STATE 9:\ \ \ \textbf{end for}
\STATE 10: (Optional) accumulate the values of $W_{si}$ between all $v_{t-1}^{(s)},s=1,2,...S$ and $x_t^{(i)},i=1,2,...N$, and obtain the $1-\alpha$ quantile as the threshold. Zero the $W_{ji}$ elements ($j,i=1,2,...N$) whose value are below the threshold.
\STATE 11: \textbf{return} $W$ \ \ // \textit{Return the main $N\times N$ matrix}
\end{algorithmic}
\end{algorithm}

Empirically, we minimize the following empirical risk:
\begin{equation}
\label{eq:empirical}
\begin{aligned}
\hat{R}_{\mathbf{X},x^{(i)},\boldsymbol{\epsilon}}[f_\theta,&\boldsymbol{\eta}]=\frac{1}{|\mathbf{T}|}\sum_{t\in \mathbf{T}}\left(x_t^{(i)}-f_\theta(\tilde{\mathbf{X}}^{(\boldsymbol{\eta})}_{t-1})\right)^2+\lambda \sum_{j=1}^{N}I(\tilde{X}^{(j)(\eta_j)}_{t-1};X^{(j)}_{t-1})
\end{aligned}
\end{equation}

In general, it may be inefficient to estimate the mutual information $I(\tilde{X}^{(j)(\eta_j)}_{t-1};X^{(j)}_{t-1})$ with large dimension of $X^{(j)}_{t-1}$ such that the expression is also differentiable w.r.t. $\eta_j$. Utilizing the property of Gaussian channels, in Appendix \ref{app:Gaussian_channel_upper_bound} we prove that $I(\tilde{X}^{(j)(\eta_j)}_{t-1};X^{(j)}_{t-1})\leq \frac{1}{2}\sum_{l=1}^{KM}\text{log}\left(1+\frac{\text{Var}(X_{t-1,l}^{(j)})}{\eta_{j,l}^2}\right)$, where $l$ denotes the $l^{\text{th}}$ element of a vector, and $\text{Var}(X_{t-1,l}^{(j)})$ is the variance of $X_{t-1,l}^{(j)}$ across $t$. Therefore, in practice to improve efficiency, we can optimize an \emph{upper bound} of the risk:
\begin{equation}
\label{eq:empirical_upper_bound}
\begin{aligned}
\hat{R}^{\text{upper}}_{\mathbf{X},x^{(i)},\boldsymbol{\epsilon}}[f_\theta,&\boldsymbol{\eta}]=\frac{1}{|\mathbf{T}|}\sum_{t\in \mathbf{T}}\left(x_t^{(i)}-f_\theta(\tilde{\mathbf{X}}^{(\boldsymbol{\eta})}_{t-1})\right)^2+\frac{\lambda}{2} \sum_{j=1}^{N}\sum_{l=1}^{KM}\text{log}\left(1+\frac{\text{Var}(X_{t-1,l}^{(j)})}{\eta_{j,l}^2}\right)
\end{aligned}
\end{equation}

When the dimension of $X_{t-1}^{(j)}$ is large, a differentiable estimate of the mutual information (e.g. MINE \cite{belghazi2018mine}) can be applied. We provide Algorithm~\ref{alg:learnable_noise} to empirically estimate $W_{ji}$, which we term relational learning with Minimum Predictive Information Regularization (MPIR). The steps 1-3 construct fake input time series $v_{t-1}^{(s)}, s=1,2,...S$ (which we know the null hypothesis of $v_{t-1}^{(s)}\independent{}x_t^{(i)}$ is true) to append to $\X_{t-1}$. Steps 4-9 optimize the objective w.r.t. the augmented dataset, and obtain a $(N+S)\times N$ matrix $W_{ji}$. Step 10 performs significance test and only preserve the $W_{ji}$ values in the main $N\times N$ matrix that are statistically significant. In the case where we only need to estimate the predictive strength, this step is not required. Finally the main matrix is returned.

To select an appropriate hyperparameter $\lambda$, we can additionally append to the target $x_t^{(i)}$ a few time series $w_t$ constructed from $\X_{t-1}$. We then select $\lambda$ such that the estimated causal strength between $\X_{t-1}$ and $w_t$ (for which we know the causal relations) is at least $4\sigma$ away from the estimated causal strength between $v_{t-1}$ and $w_t$ (for which we know that they are independent). See Appendix \ref{app:select_lambda} for details.

\section{Experiments}
\label{sec:experiments}

To demonstrate that our proposed method is able to discover interesting underlying directional (possibly causal) relations, we test it on both synthetic and real datasets. We first use synthetic datasets, where we know the underlying causal structure and compare with other methods. We then test whether our algorithm can infer directional relations among trajectories of objects from watching an agent playing video games. Finally, we apply our algorithm to a real-world heart-rate vs. breath-rate dataset and a rat EEG dataset to test its effectiveness. We use the $\hat{R}^{\text{upper}}_{\mathbf{X},x^{(i)},\boldsymbol{\epsilon}}[f_\theta,\boldsymbol{\eta}]$ (Eq. \ref{eq:empirical_upper_bound}) for optimization for all experiments.

\subsection{Synthetic experiment with log-normal causal strengths}
\label{sec:synthetic}

In this experiment, we evaluate our method together with other methods using a nonlinear synthetic dataset generated to have a known causal structure (hidden to the methods being compared). We study performance with varying number $N$ of time series, with $N$ up to $30$. To generate the data, we let each $x_t^{(i)}$ have dimension $M=1$, and also set the maximum time horizon $K=3$, so each $X_{t-1}^{(j)}$ is a $K\times M=3\times 1$ matrix. We use the following realization of the response function $h_i$ in Eq.~(\ref{eq:response_function}):
\begin{equation}
\label{eq:synthetic}
\begin{aligned}
    x_t^{(i)} = h_i(\mathbf{X}_{t-1}, u_t)=&\text{H}_1\left(\sum_{j=1}^N\left[A_{ji} \odot\text{H}_2(B_j\odot X_{t-1}^{(j)})\right]\right)+u_t,i=1,2,...N
\end{aligned}
\end{equation}
where $u_t\sim N(\mathbf{0},\mathbf{I})\in \R^M$, $\odot$ denotes element-wise multiplication, and $\text{H}_1$ and $\text{H}_2$ are two nonlinear functions to make the response functions nonlinear. In this experiment, we use $\text{H}_1(x)=\text{softplus}(x)=\text{log}(1+e^x)$, and $\text{H}_2(x)=\text{tanh}(x)$. $B_j$ is a $K\times M$ random matrix, whose element is sampled from $U[-1, 1]$. $A_{ji}$ is a $K\times M$ matrix, with 0.5 probability of being a zero matrix and 0.5 probability of being a nonzero random matrix, characterizing the underlying causal strength from $j$ to $i$. Crucially, to reflect that the causal strength may span different orders of magnitude, if $A_{ji}$ is sampled to be a nonzero matrix, then the amplitude of each of its element is sampled from a log-normal distribution with $\mu=1,\sigma=0$, their sign sampling from $U\{-1,1\}$. Denote $\mathbbm{1}(A)$ as the 0-1 indicator matrix of causality ($\mathbbm{1}(A)_{ji}=1\ \text{if}\ |A_{ji}|>0; 0\ \text{otherwise}$). The goal of each algorithm being evaluated is to produce an $N\times N$ score matrix $\tilde{A}$, where each entry $\tilde{A}_{ji}$ characterizes the directional strength from $j$ to $i$. Then the flattened $\tilde{A}$ is evaluated against the flattened $\mathbbm{1}(A)$ (excluding diagonal elements of the matrices) via different metrics. Fig. S\ref{fig:synthetic_example_figure} in Appendix \ref{app:synthetic_exp} shows example snapshots of the time series.

In general, for a large $N$, the number of possible causal graphs grows double exponentially: there are $2^{N^2}$ possible matrix of $\mathbbm{1}(A)$. To give an estimate, for $N=3,4,5,8,10,20,30$, there are $512, 6.6\times10^4, 3.3\times10^7,1.8\times10^{19}, 1.2\times10^{30}, 2.6\times 10^{120}, 8.5\times10^{270}$ number of possible graphs, respectively. Therefore, estimating the underlying causal graph is in general a non-trivial task when $N$ is large. We compare our algorithm with previous methods including transfer entropy \cite{schreiber2000measuring}, causal influence \cite{janzing2013quantifying}, linear Granger causality \cite{granger1969investigating, ding2006granger}, kernel Granger causality \cite{marinazzo2008kernel,marinazzo2008kernel2}, and three baselines: (1) mutual information $\tilde{A}_{ji}=I(X_{t-1}^{(j)};x_t^{(i)})$ (which gives $\tilde{A}_{ji}=\tilde{A}_{ij}$), (2) a sparse feature selection method, elastic net \cite{zou2005regularization}, and (3) a random matrix, each element of which is drawn from a standard Gaussian distribution. For each $N$, we sample 10 datasets with different $A_{ji}$ and $B_j$ matrices, and compare each method's average performance over 10 datasets together with their standard deviation. The implementation details for each method and each experiment are provided in Appendix \ref{app:algorithm_implementation} and \ref{app:synthetic_exp}, respectively. Since many of the methods do not provide a threshold or significance test, we use the standard metrics of area under the precision-recall curve (AUC-PR) \cite{davis2006relationship} (Table \ref{table:synthetic_larger_N_AUC_PR} below) and area under the ROC curve (AUC-ROC) (Table S\ref{table:synthetic_larger_N_AUC_ROC} in Appendix \ref{app:synthetic_ROC_AUC}) to compare their performance.

\begin{table*}[t]
\caption{Mean and standard deviation of AUC-PR (\%) vs. $N$, over 10 random sampling of datasets. Bold font marks the top method for each $N$.}
\resizebox{1\linewidth}{!}{%
\begin{tabular}{p{2.75cm}p{1cm}p{1cm}p{1cm}p{1cm}p{1cm}p{1cm}p{1cm}p{1cm}}
\toprule
 \ \ \ \ \ \ \ \ \ \ \ \ \ \ \ \ \ \ \ \ \ \ \ \ \ \ $N$ &     3  &    4  &    5  &    8  &    10 &    15 &    20 &    30 \\
method             &        &       &       &       &       &       &       &       \\
\midrule
\textbf{MPIR (ours)}        &   97.5{\tiny$\pm$5.3} &  98.4{\tiny$\pm$2.5} &  \textbf{97.6}{\tiny$\pm$2.7} &  \textbf{96.1}{\tiny$\pm$2.4} &  \textbf{93.5}{\tiny$\pm$3.7} &  \textbf{91.3}{\tiny$\pm$3.0} &  \textbf{85.9}{\tiny$\pm$2.4} &  \textbf{76.3}{\tiny$\pm$1.5} \\
Mutual Information &   90.5{\tiny$\pm$13.7} &  93.3{\tiny$\pm$3.8} &  90.0{\tiny$\pm$4.3} &  82.4{\tiny$\pm$5.1} &  76.9{\tiny$\pm$9.3} &  76.8{\tiny$\pm$4.8} &  71.9{\tiny$\pm$3.8} &  70.6{\tiny$\pm$3.1} \\
Transfer Entropy   &   93.5{\tiny$\pm$7.7} &  97.3{\tiny$\pm$3.3} &  91.6{\tiny$\pm$8.2} &  83.7{\tiny$\pm$7.2} &  76.2{\tiny$\pm$5.7} &  67.1{\tiny$\pm$4.2} &  61.2{\tiny$\pm$4.3} &   55.7{\tiny$\pm$2.5} \\
Linear Granger     &   \textbf{99.4}{\tiny$\pm$1.8} &  97.8{\tiny$\pm$2.5} &  92.0{\tiny$\pm$8.3} &  83.1{\tiny$\pm$8.8} &  79.4{\tiny$\pm$9.2} &  71.0{\tiny$\pm$10.0} &  63.7{\tiny$\pm$8.8} &  52.4{\tiny$\pm$1.7} \\
Kernel Granger     &   99.3{\tiny$\pm$2.3} &  \textbf{99.3}{\tiny$\pm$1.5} &  96.5{\tiny$\pm$4.8} &  92.5{\tiny$\pm$3.4} &  90.0{\tiny$\pm$3.3} &  86.0{\tiny$\pm$2.4} &  81.0{\tiny$\pm$4.0} &  73.1{\tiny$\pm$1.8} \\
Elastic Net        &   99.1{\tiny$\pm$2.9} &  98.5{\tiny$\pm$2.0} &  95.7{\tiny$\pm$4.2} &  88.9{\tiny$\pm$6.2} &  83.6{\tiny$\pm$4.6} &  79.1{\tiny$\pm$3.0} &  75.3{\tiny$\pm$3.6} &  69.1{\tiny$\pm$5.8} \\
Causal Influence   &   67.5{\tiny$\pm$26.7} &  60.2{\tiny$\pm$24.1} &  59.3{\tiny$\pm$15.3} &  44.1{\tiny$\pm$8.9} &  42.7{\tiny$\pm$7.8} &  47.0{\tiny$\pm$3.1} &  44.5{\tiny$\pm$4.1} &   44.6{\tiny$\pm$2.1} \\
Gaussian random    &   60.0{\tiny$\pm$14.7} &  57.9{\tiny$\pm$12.9} &  51.6{\tiny$\pm$8.0} &  44.5{\tiny$\pm$5.6} &  41.3{\tiny$\pm$6.2} &  44.6{\tiny$\pm$4.0} &  44.0{\tiny$\pm$2.4} &  44.3{\tiny$\pm$2.4} \\
\bottomrule
\end{tabular}
}
\label{table:synthetic_larger_N_AUC_PR}
\end{table*}

We see that for smaller $N$ ($N\le4$), methods with smaller expressivity (linear Granger, kernel Granger) performs slightly better. However, as $N$ becomes larger, our method outperforms other methods with increasing margin, demonstrating our method's capability to infer complex relational structures from interacting time series. Particularly, although two linear methods, linear Granger and elastic net, have relatively strong performance with $N\leq5$, they quickly degrade with larger $N$ due to more nonlinearity present in the data. With the help of kernels, kernel Granger degrades slower, but can not compete in larger $N$ with our method which allows expressive neural nets to model complex nonlinear interactions. For the Causal Influence method, although it has very good mathematical properties, it may be impractical in practice, as is also shown in the table. This is due to that it is defined as the KL-divergence between $(\mathbf{X}_{t-1}, x_{t-1}^{(i)})$ and its counterpart (whose causal arrows to and from time series $j$ are cut), each of which is an $(NK+1)M-$dimensional vector, which can quickly go to high dimensions, where density estimation required to calculate KL-divergence is in general data-hungry and difficult. In comparison, our method that estimates predictive strength via minimizing prediction errors is comparatively easier in high dimensions.

\subsection{Experiments with video games}
\label{sec:video_games}

\begin{figure}[h]
    \centering
    \begin{subfigure}{.24\linewidth}
        \includegraphics[scale=0.14]{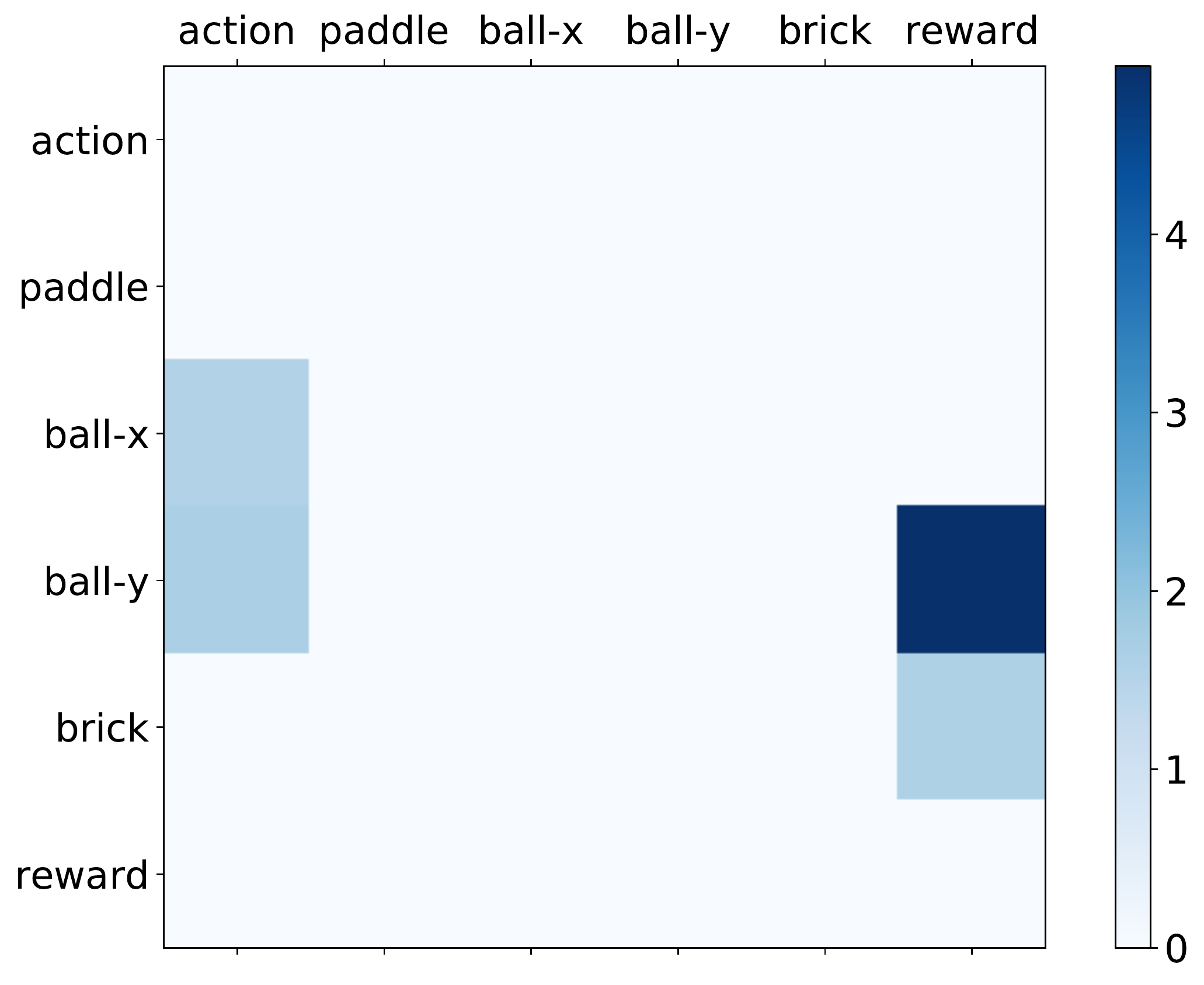}
        \caption{}
    \end{subfigure}
    \rulesep
    \begin{subfigure}{.24\linewidth}
        \includegraphics[scale=0.14]{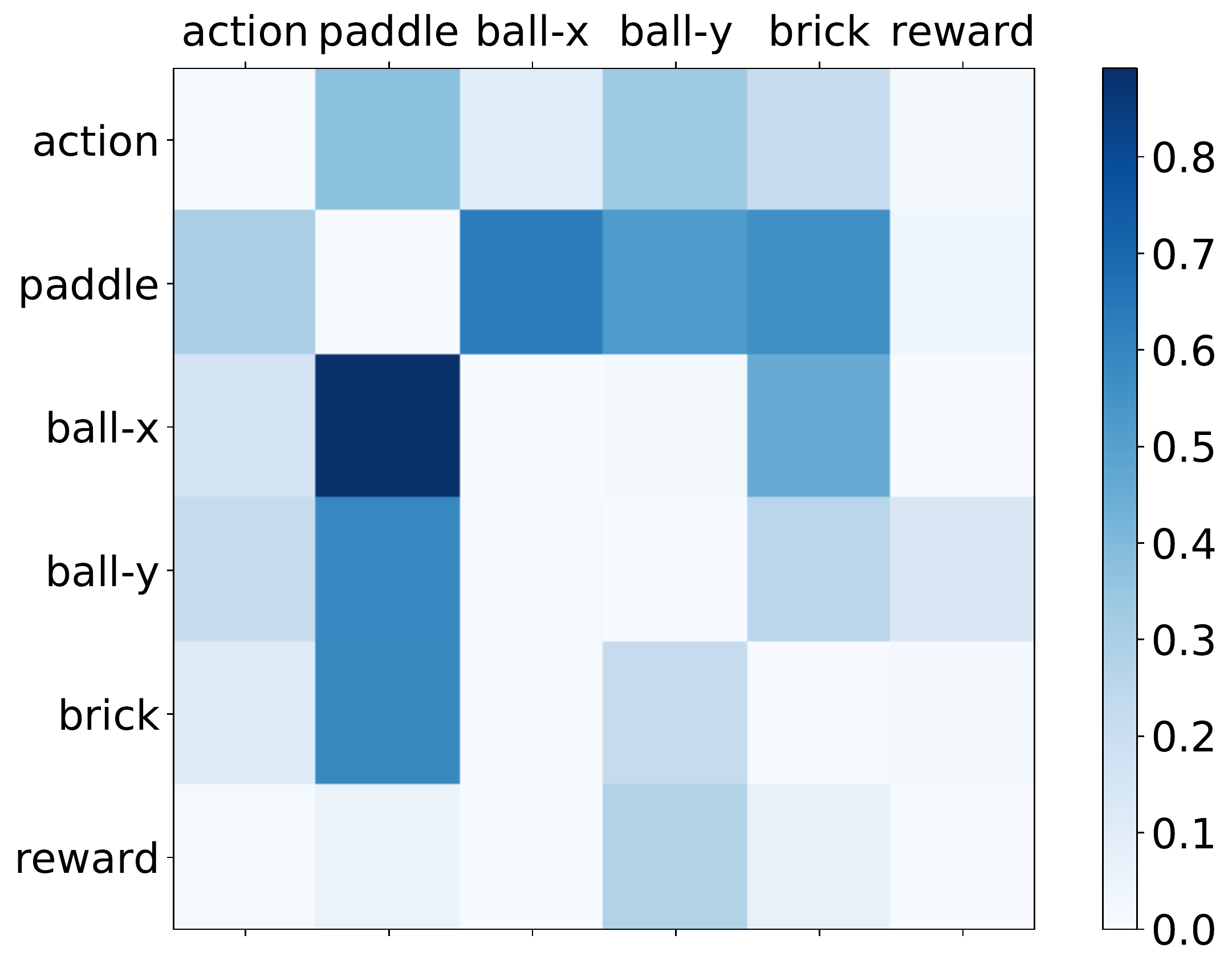}
        \caption{}
    \end{subfigure}
    \begin{subfigure}{.24\linewidth}
        \includegraphics[scale=0.14]{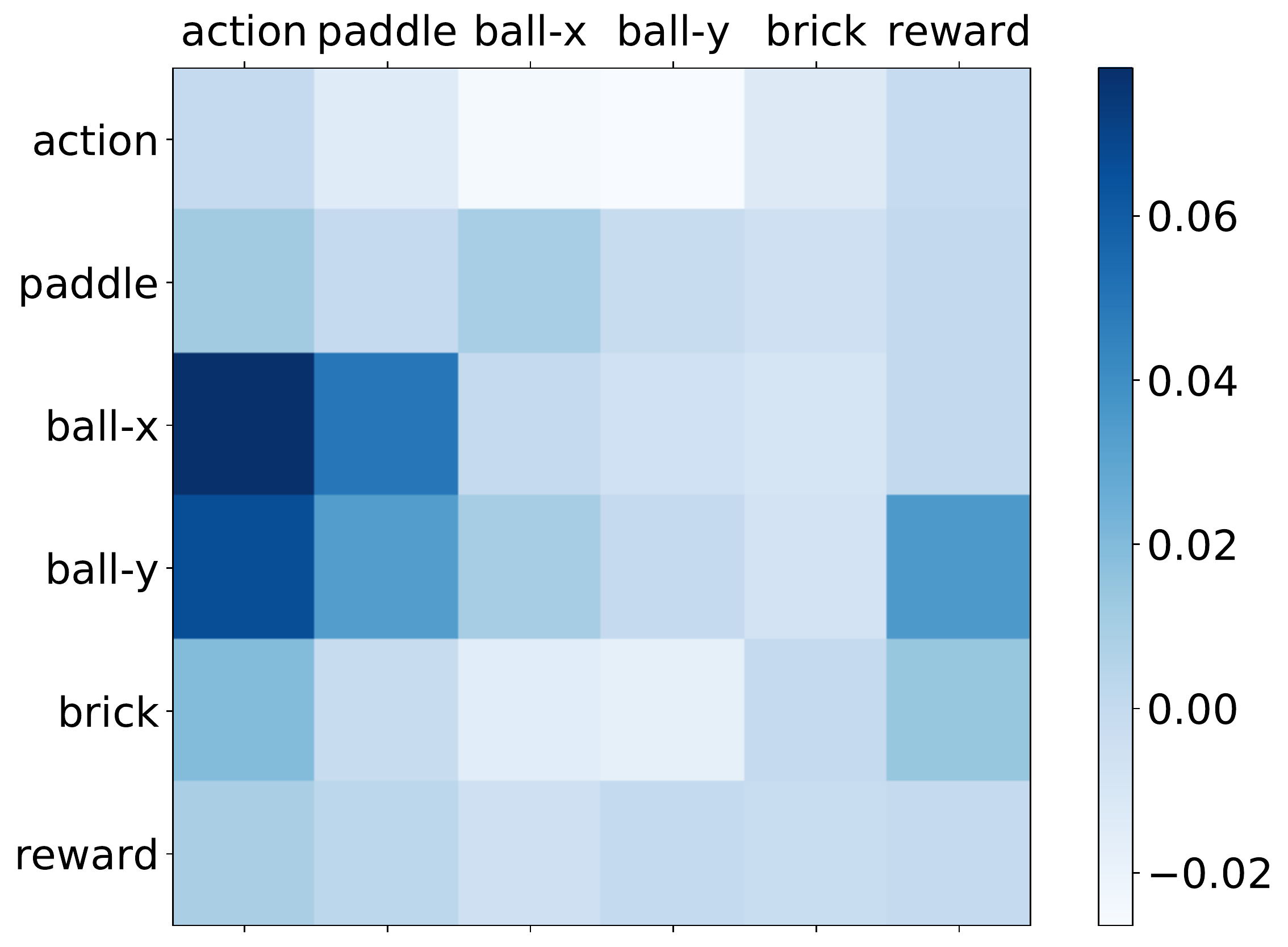}
        \caption{}
    \end{subfigure}
    \begin{subfigure}{.24\linewidth}
        \includegraphics[scale=0.14]{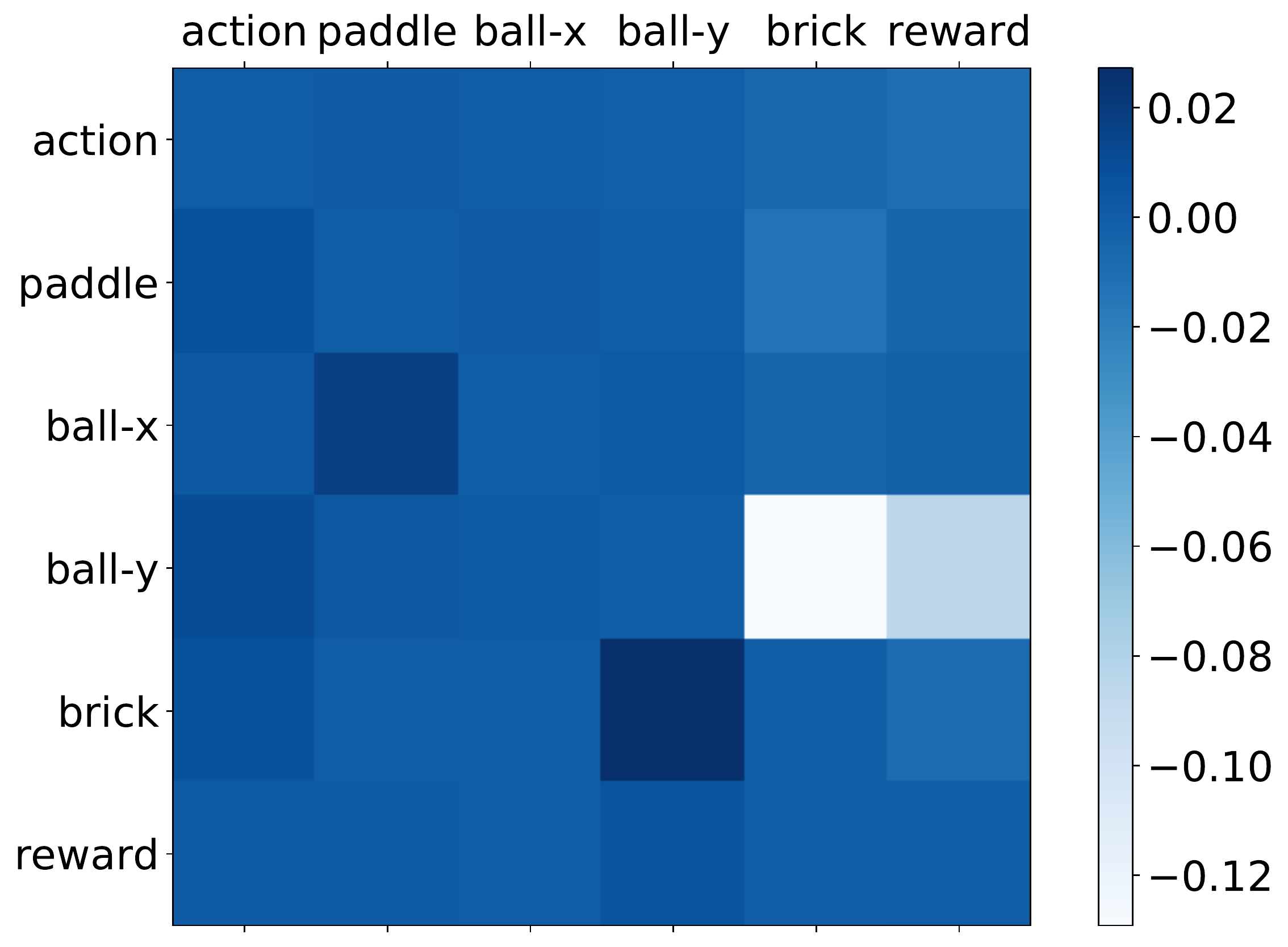}
        \caption{}
    \end{subfigure}
    \begin{subfigure}{.24\linewidth}
        \includegraphics[scale=0.14]{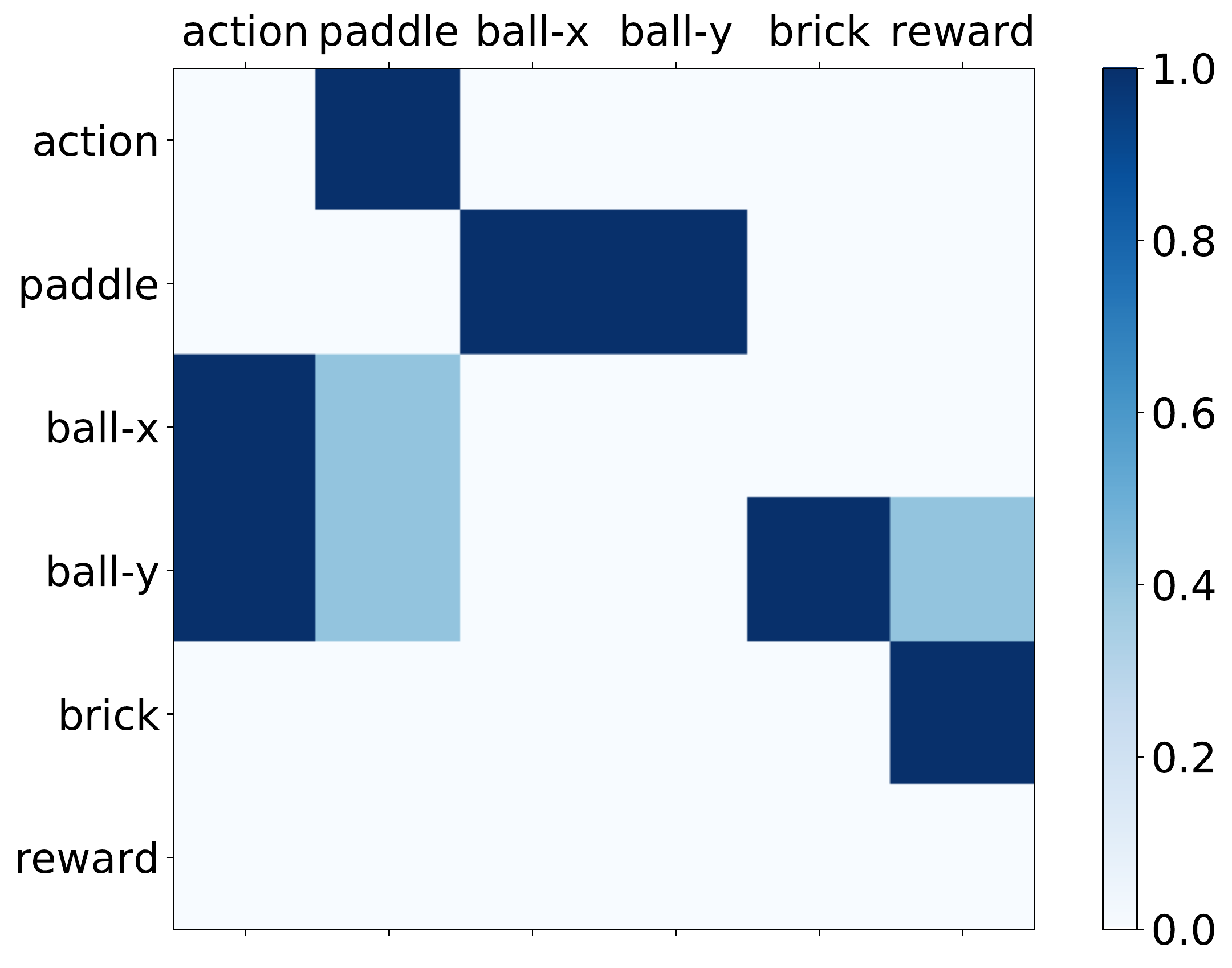}
        \caption{}
    \end{subfigure}
    \rulesep
    \begin{subfigure}{.24\linewidth}
        \includegraphics[scale=0.14]{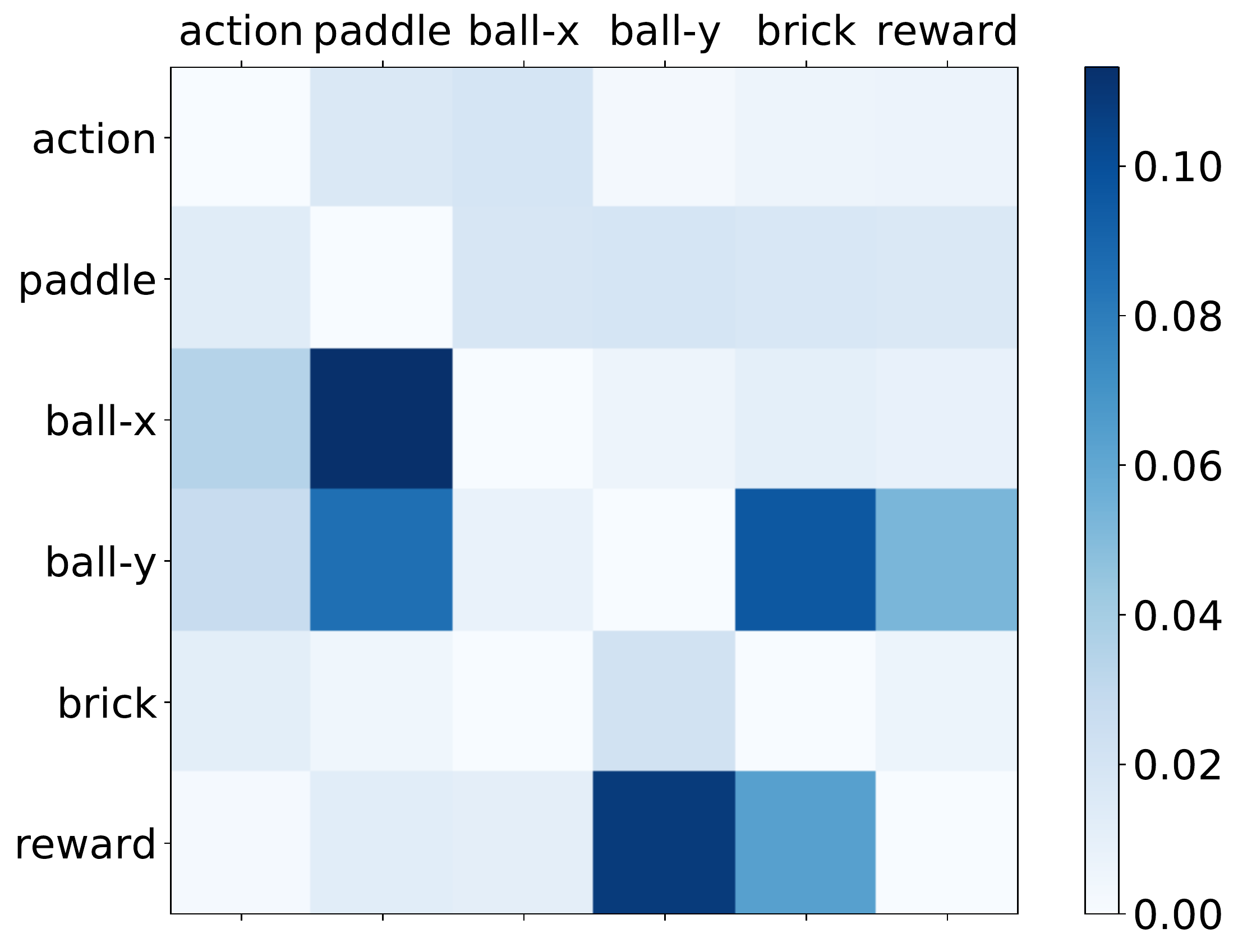}
        \caption{}
    \end{subfigure}
    \begin{subfigure}{.24\linewidth}
        \includegraphics[scale=0.14]{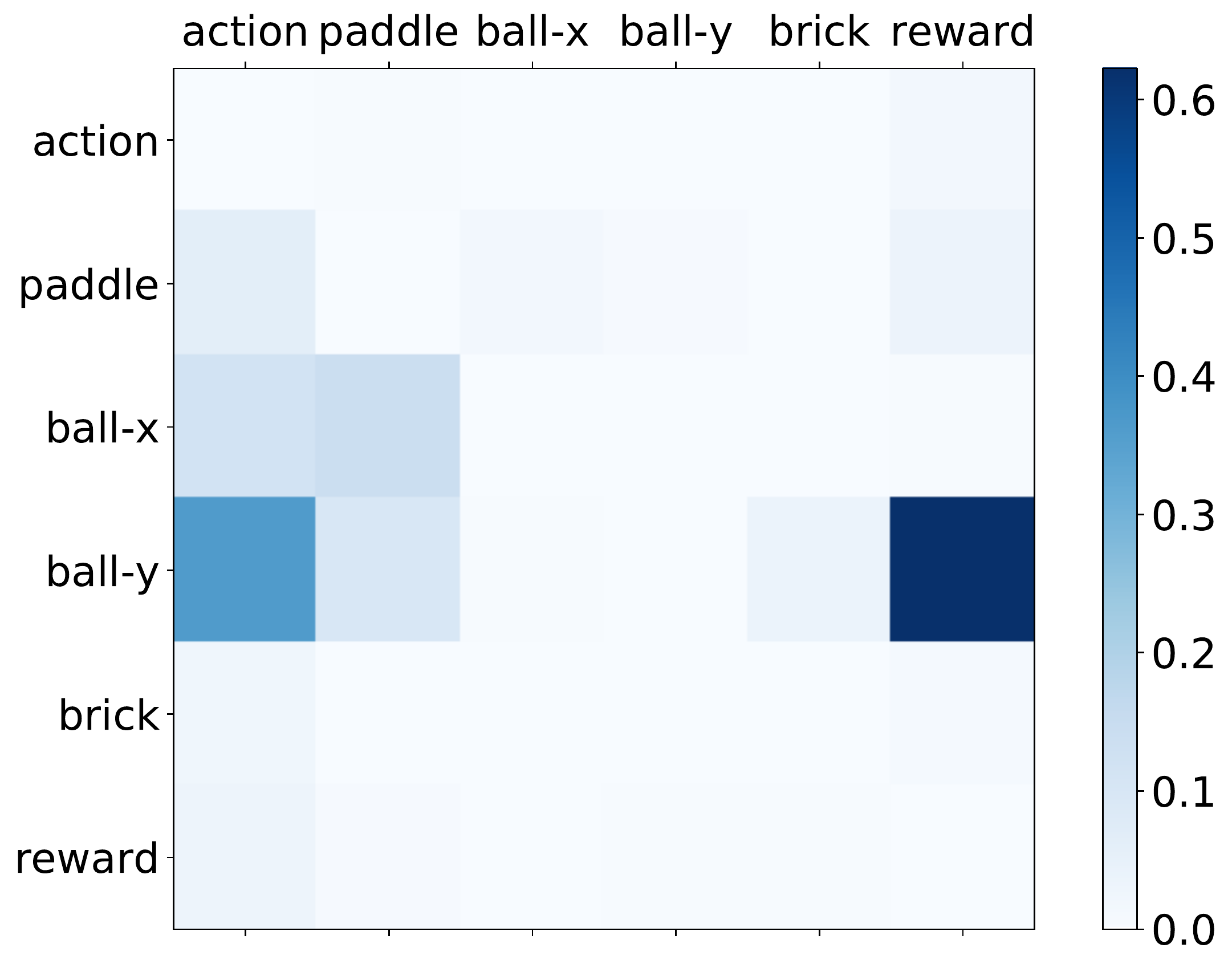}
        \caption{}
    \end{subfigure}
    \begin{subfigure}{.24\linewidth}
        \includegraphics[scale=0.14]{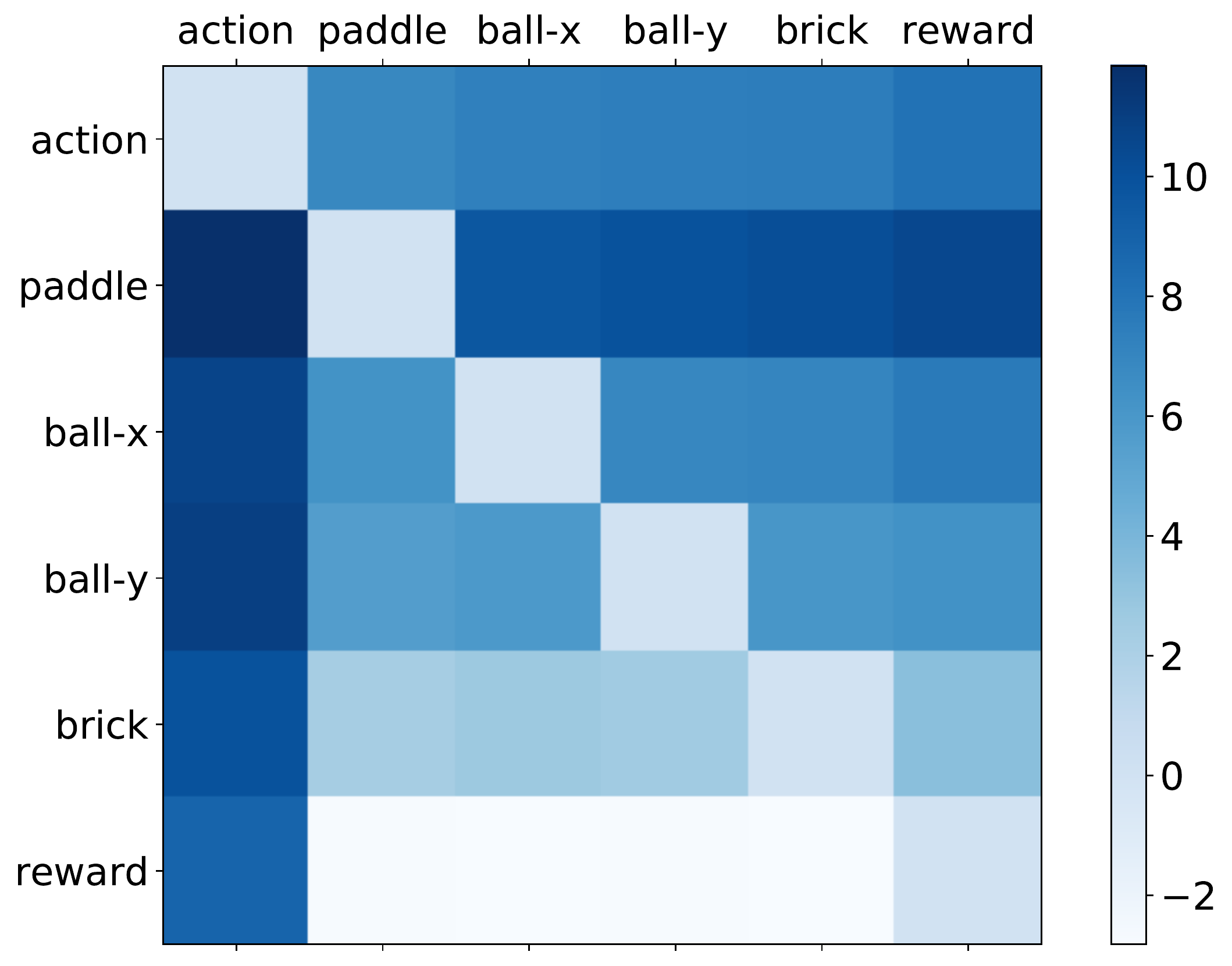}
        \caption{}
    \end{subfigure}
    \caption{(a) Predictive strength $W_{ji}$ inferred by our method in Section \ref{sec:video_games}. The $(j,i)$ element denotes the inferred causal strength from $j$ to $i$. (e) True underlying causal relations are marked dark, with light color marking competing causal relations that are indistinguishable from data. Other subfigures are: directional strength inferred by (b) mutual information (c) transfer entropy (d) linear Granger (f) kernel Granger (g) elastic net (h) causal influence.}
\label{fig:breakout_comparison}
\end{figure}

To see how our method can discover the directional (possibly causal) relations in real video games, and potentially improve reinforcement learning (RL) or imitation learning (IL), we apply our method to the relational inference between the trajectories of different objects from a trained CNN RL-agent playing Atari Breakout games (\cite{bellemare2013arcade}, implementation details see Appendix \ref{app:breakout}). Fig.~\ref{fig:breakout_comparison} shows the inferred $W_{ji}$ matrix for our method and compared methods, respectively. The true underlying causal chain is marked in dark color in Fig.~\ref{fig:breakout_comparison}e, with light color marking the competing causal relations that are indistinguishable from data (e.g. decrease of bricks and increase of reward happen at the same time step, so we cannot distinguish ball-y$\to$brick and ball-y$\to$reward). 
Compared with other methods, we see that our method is able to discover comparatively most of the causal relations without finding false positives. Specifically, it correctly discovers a prominent causal direction from the ball's $y$ position to the reward, as well as brick $\to$ reward, ball-x $\to$ action, ball-y $\to$ action. The latter two show that the ball's $x$ and $y$ positions also have influences on the trained agent's action: in order that the ball does not fall to the bottom, the agent has to position itself at the right position depending on the $x$ and $y$ positions of the ball. 

In comparison, mutual information (Fig. \ref{fig:breakout_comparison}b) gives a symmetric matrix that does not differentiate the two possible directions, and also misses the arrows ball-y$\to$brick$\to$reward. For transfer entropy (Fig. \ref{fig:breakout_comparison}c), although it correctly discovers a number of  causal arrows, it also gives relatively high scores for some incorrect arrows: brick$\to$ action, ball-y$\to$ball-x. For kernel Granger (Fig. \ref{fig:breakout_comparison}f), although it correctly discovers four causal  relations, it also incorrectly finds reward$\to$ball-y and reward$\to$brick. For elastic net (Fig. \ref{fig:breakout_comparison}g), it correctly discovers two prominent causal relations: ball-y$\to$action and ball-y$\to$reward, but misses a few others. Linear Granger (Fig. \ref{fig:breakout_comparison}d) and causal influence (Fig. \ref{fig:breakout_comparison}h) fail to discover useful causal arrows.

\subsection{Experiment with heart-rate vs. breath-rate and rat brain EEG datasets}
\label{sec:heart_rate}

\begin{figure}[b]
\centering
\begin{subfigure}[b]{.32\linewidth}
\includegraphics[scale=0.24]{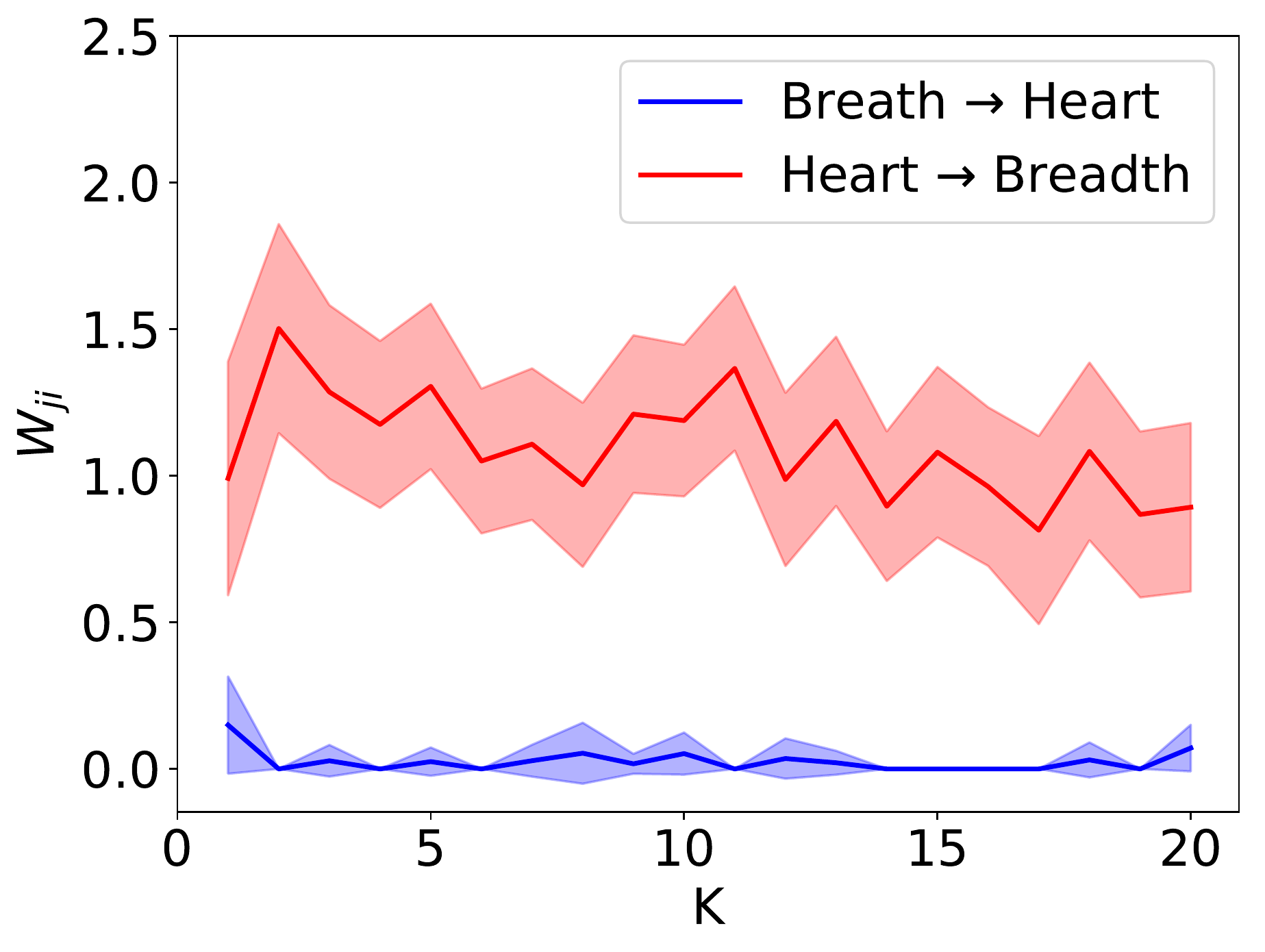}
\vspace{-15pt}\caption{}
\end{subfigure}
\begin{subfigure}[b]{.32\linewidth}
\includegraphics[scale=0.34]{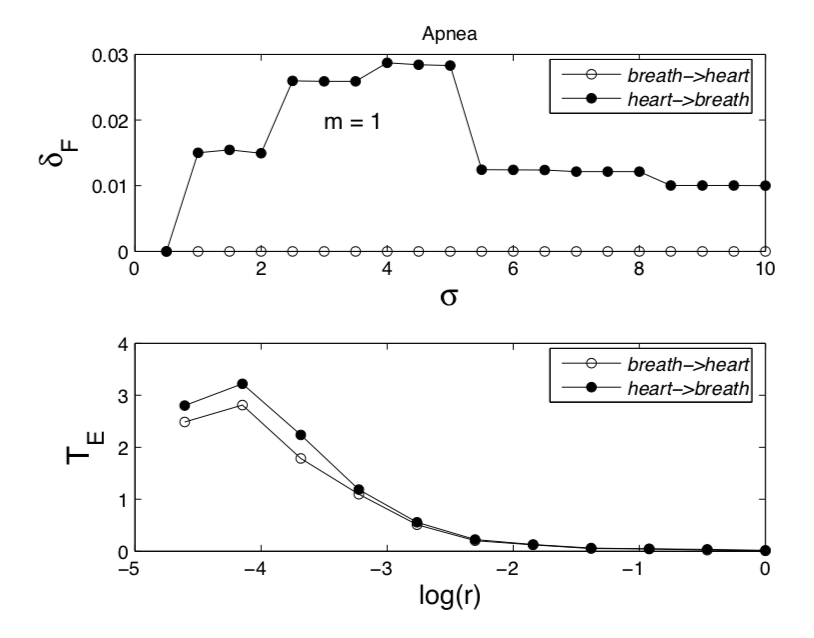}
\vspace{-15pt}\caption{}
\end{subfigure}
\begin{subfigure}[b]{.32\linewidth}
\includegraphics[scale=0.34]{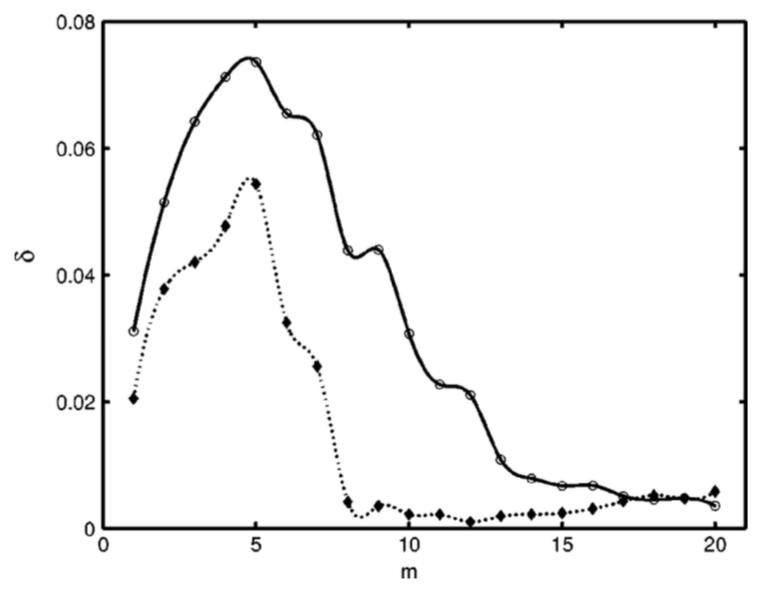}
\vspace{-15pt}\caption{}
\end{subfigure}
\vspace{-4pt}
\caption{(a) Predictive strength $W_{ji}$ inferred by our method with the heart-rate vs. breath-rate dataset, averaged over 50 initializations of $f_\theta$. The shaded areas are the 95\% confidence interval. (b) Upper: the filtered causality index  vs. varying width of Gaussian kernel $\sigma$ \cite{marinazzo2008kernel}; lower: transfer entropy vs. $r$, the length scale \cite{schreiber2000measuring}; (c) The causality index for breath$\to$heart (lower) and heart$\to$breath (upper) in \cite{ancona2004radial}, where $m$ is the maximum time lag (equivalent to our $K$).}
\label{fig:apnea}%
\vskip -0.08in
\end{figure}

Now we test our algorithm with real-world datasets. As a common dataset studied in previous causal works, we use the time-series of the breath rate and instantaneous heart rate of a sleeping patient suffering from sleep apnea (samples 2350-3550 of data set B from Santa Fe Institute time series contest held in 1991, available in \cite{Physionet}). We apply our method to infer the directional relations between the breath rate and heart rate, with different maximum time horizon $K$. The result is shown in Fig.~\ref{fig:apnea}. We see that the predictive strength $W_{ji}$ from heart to breath is significantly higher than the reverse direction that is basically 0, consistent with the results from previous causal inference methods \cite{schreiber2000measuring, ancona2004radial,marinazzo2008kernel} as also shown in Fig.~\ref{fig:apnea}(b)(c). Notably, the $W_{ji}$ from heart to breath estimated by our method remains at roughly the same level for different $K$s, in contrast to the decaying causality index w.r.t. increasing history length in (\cite{ancona2004radial}, Fig. \ref{fig:apnea} (c)), showing a merit of our method in estimating directional strength across different time-horizons, aided by the flexibility of neural nets in extracting the right information to predict the future. The implementation details are provided in Appendix \ref{app:real_dataset}. In addition, in Appendix \ref{app:rat_EEG_experiment} we test our algorithm on a rat EEG dataset, and obtain consistent result with previous works.

\section{Discussion and conclusion}

In this paper, we have introduced a novel relational learning with Minimum Predictive Information Regularization (MPIR) method for exploratory discovery of nonlinear directional relations from observational time series. It allows functional approximators like neural nets to learn complex directional relations from time series data. We prove its three theoretical properties,
and provide intuition that it favors variables that directly cause the variable of interest. We demonstrate in synthetic datasets, a video game environment and heart-rate vs. breath-rate dataset, that our method has better capability to handle nonlinearity, and can scale to large numbers of time series. We believe our work endows practitioners with a useful tool for deciphering the directional relations in complex systems, and are excited to see it in broader applications.

\bibliography{reference}
\bibliographystyle{icml2019}

\newpage
\onecolumn
\appendix
\begin{center}
\begin{huge}
\textbf{Appendix}
\end{huge}
\end{center}

\section{Hyperparameter $\lambda$ selection}
\label{app:select_lambda}

For selecting an appropriate hyperparameter $\lambda$, we run our experiments for the synthetic dataset with $\lambda=0.001,0.002,0.005,0.01,0.02,0.05$. For each experiment involving $N$ time series, we append $\ceil[\big]{N/2}$  independent time series $v_{t-1}^{(s)}$ ($s=1,2,...\ceil[\big]{N/2}$) to $\X_{t-1}$, generated by randomly sampling $\ceil[\big]{N/2}$ time series from $\X_{t-1}$ and performing random permutation across the examples. We also append $\ceil[\big]{N/2}$ time series $w_t^{(i)}, i=1,2,...\ceil[\big]{N/2}$ to $x_t^{(i)}$, such that $w_t^{(i)}=X_{t-1}^{(i)}\cdot Q$, where $Q$ is a fixed random  $K\times1$ matrix, so that we know $X_{t-1}^{(i)}$ causes $w_{t}^{(i)}$, and $v_{t-1}^{(s)}$ does not cause $w_{t}^{(i)}$ for any $i,s$. We apply Alg. \ref{alg:learnable_noise} to the augmented dataset, and produce the estimated predictive strength $W_{ji}$ from $[\X_{t-1},\mathbf{v}_{t-1}]$ to $[\mathbf{x}_t,\mathbf{w}_t]$. For each hyperparameter $\lambda$, we then fit a Gaussian distribution $G_{v\to w}$ to the estimated predictive strengths from $v_{t-1}^{(s)}$ to $w_{t}^{(i)}$ ($s=1,2,...\ceil[\big]{N/2};j=1,2,...\ceil[\big]{N/2}$), and fit another Gaussian distribution $G_{x\to w}$ to the estimated predictive strengths from $X_{t-1}^{(i)}$ to $w_{t}^{(i)}$, $i=1,2,...\ceil[\big]{N/2}$, and select the $\lambda$ such that the upper $4\sigma$ value of $G_{v\to w}$ is smaller than the lower $4\sigma$ value of $G_{x\to w}$. In this way, for the known causal and non-causal relations, they are sufficiently apart. We find that $\lambda=0.001$ and $\lambda=0.002$ satisfy this criterion, while larger $\lambda$ fails to satisfy. We then set $\lambda=0.002$ for all our experiments.

\section{Proof and analysis of the Minimum Predictive Information regularized risk}
\label{app:W_proof}

In this section we prove the three properties of $W_{ji}$ in Section \ref{sec:our_method}, and analyze why it is likely to select variables that directly causes the variable of interest.

Firstly we state the assumption that will be used throughout this section:

\begin{assumption}
\label{assump:f_theta}
Assume that $f_\theta\in\mathcal{F}$ is a continuous function and has enough capacity so that it can approximate any $\int dx^{(i)}_t P(x_t^{(i)}|\mathbf{X}_{t-1})x_t^{(i)}$.  Let $j\neq i$ and assume that $P(X^{(j)}_{t-1})$ has support with intrinsic dimension of $KM$.
\end{assumption}

Also we emphasize that in this paper, the expected risks (with symbol $R$) are w.r.t. the distributions, and the empirical risks (with symbol $\hat{R}$) are w.r.t. a dataset drawn from the distribution, with finite number of examples. The theorems in this paper are all proved w.r.t. distributions (assuming infinite number of examples). Sample complexity results will be left for future work.

Before going forward with the main proof, we first prove the following lemma.

\subsection{Proving a lemma}

\begin{lemma}
\label{lemma:argmin_risk}
Suppose that Assumption \ref{assump:f_theta} holds. Denote

$$R_{\mathbf{X},x^{(i)}}^{\text{MSE}}[f_{\theta}]=\mathbb{E}_{\mathbf{X}_{t-1},x_t^{(i)}}\left[\left(x_t^{(i)}-f_\theta(\mathbf{X}_{t-1})\right)^2\right]$$ as the standard MSE loss, we have
\begin{equation}
\text{argmin}_{f_\theta}R^{\text{MSE}}_{\mathbf{X},x^{(i)}}[f_{\theta}]=\int dx^{(i)}_t P(x_t^{(i)}|\mathbf{X}_{t-1})x_t^{(i)}
\end{equation}

and
\begin{equation}
\label{eq:min_mse_risk}
\text{min}_{f_\theta}R^{\text{MSE}}_{\mathbf{X},x^{(i)}}[f_{\theta}]=\mathbb{E}_{\mathbf{X}_{t-1},x_t^{(i)}}\left[\left(x_t^{(i)}-\int dx^{(i)}_t P(x_t^{(i)}|\mathbf{X}_{t-1})x_t^{(i)}\right)^2\right]
\end{equation}

In other words, for the MSE loss, its minimum is attained when $f_\theta(\mathbf{X}_{t-1})$ is the expectation of $x_t^{(i)}$ conditioned on  $\mathbf{X}_{t-1}$.

\end{lemma}

\begin{proof}

The proof of the lemma is adapted from \cite{papoulis1985probability}. The risk
\begin{equation*}
\begin{aligned}
R^{\text{MSE}}_{\mathbf{X},x^{(i)}}[f_\theta]&=\mathbb{E}_{\mathbf{X}_{t-1},x_t^{(i)}}\left[\left(x_t^{(i)}-f_\theta(\mathbf{X}_{t-1})\right)^2\right]\\
&=\int d\mathbf{X}_{t-1} dx^{(i)}_t \cdot P(\mathbf{X}_{t-1},x^{(i)}_t) \left(x_t^{(i)}-f_\theta(\mathbf{X}_{t-1})\right)^2 \\
&=\int d\mathbf{X}_{t-1} P(\mathbf{X}_{t-1})\int dx^{(i)}_t P(x_t^{(i)}|\mathbf{X}_{t-1})\left(x_t^{(i)}-f_\theta(\mathbf{X}_{t-1})\right)^2
\end{aligned}
\end{equation*}

Note that here  $(x_t^{(i)}-f_\theta(\mathbf{X}_{t-1}))^2\equiv\big\langle x_t^{(i)}-f_\theta(\mathbf{X}_{t-1}), x_t^{(i)}-f_\theta(\mathbf{X}_{t-1}) \big\rangle$ is an inner product in $\R^M$.

For any $\mathbf{X}_{t-1}$, treating $f_\theta(\mathbf{X}_{t-1})\in \R^M$ as a vector, let's calculate its value such that the integral $F(f_\theta(\mathbf{X}_{t-1})):=\int dx^{(i)}_t P(x_t^{(i)}|\mathbf{X}_{t-1})\left(x_t^{(i)}-f_\theta(\mathbf{X}_{t-1})\right)^2$ attains its minimum.

Let
\begin{equation*}
\begin{aligned}
0&=\frac{\partial}{\partial f_{\theta}(\mathbf{X}_{t-1})}F(f_\theta(\mathbf{X}_{t-1}))\\
&=\frac{\partial}{\partial f_{\theta}(\mathbf{X}_{t-1})}\int dx^{(i)}_t P(x_t^{(i)}|\mathbf{X}_{t-1})\left(x_t^{(i)}-f_\theta(\mathbf{X}_{t-1})\right)^2 \\
&= -2\int dx^{(i)}_t P(x_t^{(i)}|\mathbf{X}_{t-1})\left(x_t^{(i)}-f_\theta(\mathbf{X}_{t-1})\right)
\end{aligned}
\end{equation*}

we have
\begin{equation*}
\begin{aligned}
\int dx^{(i)}_t P(x_t^{(i)}|\mathbf{X}_{t-1})x_t^{(i)}&=\int dx^{(i)}_t P(x_t^{(i)}|\mathbf{X}_{t-1})f_\theta(\mathbf{X}_{t-1})\\
&=f_\theta(\mathbf{X}_{t-1})\int dx^{(i)}_t P(x_t^{(i)}|\mathbf{X}_{t-1})\\
&=f_\theta(\mathbf{X}_{t-1})
\end{aligned}
\end{equation*}

Therefore, for any $\mathbf{X}_{t-1}$, $f_\theta(\mathbf{X}_{t-1})=\int dx^{(i)}_t P(x_t^{(i)}|\mathbf{X}_{t-1})x_t^{(i)}$ is the only stationary point for $F(f_\theta(\mathbf{X}_{t-1}))$.

Taking the second derivative, we have
\begin{equation*}
\begin{aligned}
&\frac{\partial^2}{(\partial f_{\theta}(\mathbf{X}_{t-1}))^2}F(f_\theta(\mathbf{X}_{t-1})) = 2\int dx^{(i)}_t P(x_t^{(i)}|\mathbf{X}_{t-1})\mathbf{I} =2\mathbf{I}
\end{aligned}
\end{equation*}
where $\mathbf{I}$ is an $M\times M$ identity matrix, which is always positive definite.

Therefore, for any $\mathbf{X}_{t-1}$, $f_\theta(\mathbf{X}_{t-1})=\int dx^{(i)}_t P(x_t^{(i)}|\mathbf{X}_{t-1})x_t^{(i)}$ is the only global minimum of $F(f_\theta(\mathbf{X}_{t-1}))$ w.r.t. $f_\theta(\mathbf{X}_{t-1})$.

Since 
\begin{equation*}
R^{\text{MSE}}_{\mathbf{X},x^{(i)}}[f_\theta]=\int d\mathbf{X}_{t-1} P(\mathbf{X}_{t-1})F(f_\theta(\mathbf{X}_{t-1}))
\end{equation*}

The minimum of the risk $R_{\mathbf{X},x^{(i)}}[f_\theta]$ is attained iff $F(f_\theta(\mathbf{X}_{t-1}))$ attains minimum at every $\mathbf{X}_{t-1}$, i.e.,
\begin{equation*}
    f_\theta(\mathbf{X}_{t-1})=\int dx^{(i)}_t P(x_t^{(i)}|\mathbf{X}_{t-1})x_t^{(i)}
\end{equation*}
is true for any $\mathbf{X}_{t-1}$. Given Assumption \ref{assump:f_theta}, we know that $f_\theta\in \mathcal{F}$ has enough capacity such that it can approximate any $\int dx^{(i)}_t P(x_t^{(i)}|\mathbf{X}_{t-1})x_t^{(i)}$. Therefore,
\begin{equation*}
\text{argmin}_{f_\theta}R^{\text{MSE}}_{\mathbf{X},x^{(i)}}[f_{\theta}]=\int dx^{(i)}_t P(x_t^{(i)}|\mathbf{X}_{t-1})x_t^{(i)}
\end{equation*}

and
\begin{equation*}
\text{min}_{f_\theta}R^{\text{MSE}}_{\mathbf{X},x^{(i)}}[f_{\theta}]=\mathbb{E}_{\mathbf{X}_{t-1},x_t^{(i)}}\left[\left(x_t^{(i)}-\int dx^{(i)}_t P(x_t^{(i)}|\mathbf{X}_{t-1})x_t^{(i)}\right)^2\right]
\end{equation*}
\end{proof}

\subsection{Proof of the three properties of $W_{ji}$}

The three properties are 
\begin{enumerate}[label={(\arabic*)}]
\item If $x^{(j)}\independent x^{(i)}$, then $W_{ji}=0$.
\item $W_{ji}$ is invariant to affine transformation of each individual $X_{t-1}^{(k)},k=1,2,...N$.
\item $W_{ji}$ is invariant to reparameterization of $\theta$ in $f_\theta$ (the mapping remains the same).
\end{enumerate}

\begin{proof}
\textbf{(1)} If $x^{(j)}\independent x^{(i)}$, then $X^{(j)}_{t-1}\independent x^{(i)}_t$. Since $\tilde{X}^{(j)(\eta_j)}_{t-1}=X^{(j)}_{t-1}+\eta_j\cdot\epsilon_j$ where $\epsilon_j\sim N(\mathbf{0},\mathbf{I})$, we have
$\tilde{X}^{(j)(\eta_j)}_{t-1}\independent x^{(i)}_t$. Recall Eq. (\ref{eq:learnable_risk}):
$$R_{\mathbf{X},x^{(i)}}[f_\theta,\boldsymbol{\eta}]=\mathbb{E}_{\mathbf{X}_{t-1},x_t^{(i)},\boldsymbol{\epsilon}}\left[\left(x_t^{(i)}-f_\theta(\tilde{\mathbf{X}}^{(\boldsymbol{\eta})}_{t-1})\right)^2\right]+\lambda\cdot \sum_{k=1}^{N}I(\tilde{X}^{(k)(\eta_k)}_{t-1};X^{(k)}_{t-1})$$
let $f_{\theta^*_{\boldsymbol{\eta}}}=\text{argmin}_{f_{\theta}}R_{\mathbf{X},x^{(i)}}[f_\theta,\boldsymbol{\eta}]$ given a certain $\boldsymbol{\eta}$, we have

\begin{equation*}
\begin{aligned}
f_{\theta^*_{\boldsymbol{\eta}}}(\tilde{\X}_{t-1}^{(\boldsymbol{\eta})})&=\text{argmin}_{f_{\theta}}R_{\mathbf{X},x^{(i)}}[f_\theta,\boldsymbol{\eta}]\\
&=\text{argmin}_{f_{\theta}}\mathbb{E}_{\tilde{\X}_{t-1}^{(\boldsymbol{\eta})},x_t^{(i)}}\left[\left(x_t^{(i)}-f_\theta(\tilde{\mathbf{X}}^{(\boldsymbol{\eta})}_{t-1})\right)^2\right]\\
&=\int dx^{(i)}_t P(x_t^{(i)}|\tilde{\X}_{t-1}^{(\boldsymbol{\eta})})x_t^{(i)}
\end{aligned}
\end{equation*}

where the second equality is due to that the mutual information term in $R_{\mathbf{X},x^{(i)}}[f_\theta,\boldsymbol{\eta}]$ does not depend on $f_\theta$, and the last equality is due to Lemma \ref{lemma:argmin_risk}. Let $\tilde{\X}_{t-1}^{(\boldsymbol{\eta})(\hat{j})}=\tilde{\X}_{t-1}^{(\boldsymbol{\eta})}\texttt{\textbackslash}\tilde{X}^{(j)(\eta_j)}_{t-1}$, since $\tilde{X}_{t-1}^{(j)(\eta_j)}\independent x^{(i)}_t$, we have

\begin{equation*}
\begin{aligned}
P(x_t^{(i)}|\tilde{\X}_{t-1}^{(\boldsymbol{\eta})})&\equiv P(x_t^{(i)}|\tilde{\X}_{t-1}^{(\boldsymbol{\eta})(\hat{j})},\tilde{X}_{t-1}^{(j)(\eta_j)})\\
&=P(x_t^{(i)}|\tilde{\X}_{t-1}^{(\boldsymbol{\eta})(\hat{j})})
\end{aligned}
\end{equation*}

Therefore,

\begin{equation*}
\begin{aligned}
f_{\theta^*_{\boldsymbol{\eta}}}(\tilde{\mathbf{X}}_{t-1}^{(\boldsymbol{\eta})})=\int dx^{(i)}_t P(x_t^{(i)}|\tilde{\X}_{t-1}^{(\boldsymbol{\eta})(\hat{j})})x_t^{(i)}
\end{aligned}
\end{equation*}

which \emph{does not} depend on $\tilde{X}_{t-1}^{(j)(\eta_j)}$. Finally, we have

\begin{equation*}
\begin{aligned}
&\text{min}_{(f_{\theta},\boldsymbol{\eta})}R_{\mathbf{X},x^{(i)}}[f_\theta,\boldsymbol{\eta}]\\
=&\text{min}_{\boldsymbol{\eta}}\left[R_{\mathbf{X},x^{(i)}}[f_{\theta^*_{\boldsymbol{\eta}}},\boldsymbol{\eta}]\right]\\
=&\text{min}_{\boldsymbol{\eta}}\left[\mathbb{E}_{\mathbf{X}_{t-1},x_t^{(i)},\boldsymbol{\epsilon}}\left[\left(x_t^{(i)}-f_{\theta_{\boldsymbol{\eta}}^*}(\tilde{\mathbf{X}}^{(\boldsymbol{\eta})}_{t-1})\right)^2\right]+\lambda\cdot \sum_{k=1}^{N}I(\tilde{X}^{(k)(\eta_k)}_{t-1};X^{(k)}_{t-1})\right]\\
=&\text{min}_{\boldsymbol{\eta}}\left[\left(\mathbb{E}_{\mathbf{X}_{t-1},x_t^{(i)},\boldsymbol{\epsilon}}\left[\left(x_t^{(i)}-f_{\theta_{\boldsymbol{\eta}}^*}(\tilde{\mathbf{X}}^{(\boldsymbol{\eta})(\hat{j})}_{t-1})\right)^2\right]+\lambda\cdot \sum_{k\neq j}I(\tilde{X}^{(k)(\eta_k)}_{t-1};X^{(k)}_{t-1})\right)+I(\tilde{X}^{(j)(\eta_j)}_{t-1};X^{(j)}_{t-1})\right]\\
\end{aligned}
\end{equation*}

For the last equality, the elements in the parenthesis $(\cdot)$ does not depend on $\tilde{X}_{t-1}^{(j)(\eta_j)}$, and only the $I(\tilde{X}^{(j)(\eta_j)}_{t-1};X^{(j)}_{t-1})$ term depends on $\tilde{X}_{t-1}^{(j)(\eta_j)}$. Therefore, at the minimization of the whole objective $R_{\mathbf{X},x^{(i)}}[f_\theta,\boldsymbol{\eta}]$, we have $I(\tilde{X}^{(j)(\eta_j)}_{t-1};X^{(j)}_{t-1})$ attains its minimum of 0, at which $\eta_j^*\to\infty$. By the definition of $W_{ji}$, we have $W_{ji}=I(\tilde{X}^{(j)(\eta_j^*)}_{t-1};X^{(j)}_{t-1})=0$. Proof completes.

In essence, the proof states that if $x^{(j)}\independent{}x^{(i)}$, then at the minimization of the whole objective, the MSE term does not depend on $X^{(j)}_{t-1}$ or $\tilde{X}^{(j)(\eta_j)}_{t-1}$, and the mutual information term $I(\tilde{X}^{(j)(\eta_j^*)}_{t-1};X^{(j)}_{t-1})$ w.r.t. time series $j$ can be independently minimized and approach 0.

\textbf{(2)} Suppose that we replace $X_{t-1}^{(j)}$ by $X_{t-1}^{'(j)}=a\cdot X_{t-1}^{(j)}+b$ where $a,b\in\mathbb{R}$. Let $\eta_j'=a\cdot \eta_j$. We have $\tilde{X}_{t-1}^{'(j)(\eta'_j)}=X_{t-1}^{'(j)}+\eta_j'\cdot\epsilon_j=a(X_{t-1}^{(j)}+\eta_j\cdot\epsilon_j) + b=a\cdot \tilde{X}_{t-1}^{(j)(\eta_j)}+b$, and therefore $I\left(\tilde{X}_{t-1}^{'(j)(\eta'_j)};X_{t-1}^{'(j)}\right)=I\left(a\cdot \tilde{X}_{t-1}^{(j)(\eta_j)}+b;a\cdot X_{t-1}^{(j)}+b\right)=I\left(\tilde{X}_{t-1}^{(j)(\eta_j)};X_{t-1}^{(j)}\right)$, where the last equality is due to that mutual information is invariant to invertible transformations. Furthermore, due to Assumption \ref{assump:f_theta}, we can find another $f_{\theta'}$ which undoes this affine transformation on $\tilde{X}_{t-1}^{(j)(\eta_j)}$, so the MSE term can be kept the same. Therefore, we have a one-to-one mapping between the original $X_{t-1}^{(j)},\eta_j,f_\theta$ and the new $X_{t-1}^{'(j)},\eta'_j,f_{\theta'}$ such that value of the MSE term and the mutual information term remain unchanged. Thus at the minimization of the objective, $W_{ji}$ remains the same.

\textbf{(3)} This is trivial to prove. We see that in $R_{\mathbf{X},x^{(i)}}[f_\theta,\boldsymbol{\eta}]$, the MSE term remains the same if the mapping $f$ remains the same, regardless of how we parameterize $f$ in terms of parameter $\theta$. The second term does not depend on $f_\theta$. Therefore, at the minimization of $R_{\mathbf{X},x^{(i)}}[f_\theta,\boldsymbol{\eta}]$, the $W_{ji}=I(\tilde{X}^{(j)(\eta_j^*)}_{t-1};X^{(j)}_{t-1})$ is invariant to the reparameterization of the same $f$ in terms of parameter $\theta$. As a direct corollary, $W_{ji}$ is insensitive to the network architecture, as long as the capacity is enough (provided with sufficient number of examples). This is confirmed in Table S\ref{table:synthetic_capacity} in Appendix \ref{app:capacity}. 

Note that L1 and group L1 regularization do not have this property, since they explicitly regularize on the parameter $\theta$.
\end{proof}

\subsection{Analysis of the minimum predictive information-regularized risk}

After proving the three properties of $W_{ji}$, now we analyze why the minimum predictive information-regularized risk is likely to select the variables that directly cause $x_t^{(i)}$, under some additional assumptions. We first state the additional assumption needed to perform the analysis, then we restate the definitions of direct causality to make our statements more rigorous. We then prove two lemmas in Appendix \ref{app:lemma_2}, and finally perform the analysis in Appendix \ref{app:analysis_risk}.

\begin{assumption}
\label{assump:additonal_assumption}
Assume that causal sufficiency \cite{peters2017elements} is satisfied, i.e. the observed time series $x^{(i)},i=1,2,...N$ are all the variables that take part in the dynamics (no hidden confounding variables). Also assume that in the response function Eq. (\ref{eq:response_function}), the noise variable $u_i,i=1,2,...N$ are effective variables, so each $h_i$ is not a deterministic mapping. Assume that by saying ``causality", we mean ``causality in mean".
\end{assumption}

To make our statement of causality more rigorous, here we restate the definition of direct (structural) causality \cite{white2011linking} using our notations of the system Eq.~(\ref{eq:response_function}). This definition is a natural extension to Pearl causality \cite{pearl2009causality} in canonical settable systems \cite{white2009settable,white2011linking}, which formalizes time series in its full generality.

\textbf{Direct (structural) causality} \cite{white2011linking} \textit{
We say $X^{(j)}_{t-1}, j\neq i$ does not directly (structurally) cause $x^{(i)}_{t}$, if for all possible values of $\mathbf{X}_{t-1}^{(\hat{j})}$ and $u_l$, $l\in{1,2,...N}$, the function $X^{(j)}_{t-1}\to h_i(\mathbf{X}_{t-1},u_i)$ is constant in $X^{(j)}_{t-1}$. Otherwise, we say $X^{(j)}_{t-1}$ directly (structurally) causes $x^{(i)}_{t}$.
}

The relationship between direct causality and Granger causality in Section \ref{sec:problem_definition} is the following Lemma, which states that for our system, Granger causality is a sufficient condition for direct (structural) causality.

\begin{lemma}
\label{thm:theorem_0}
Assuming causal sufficiency, for system Eq.~\ref{eq:response_function}, for any $i,j\in\{1,2,...N\},i\neq j$, if $X^{(j)}_{t-1}$ Granger-causes $x^{(i)}_{t}$, then $X^{(j)}_{t-1}$ directly structurally causes $x^{(i)}_t$.
\end{lemma}

\begin{proof}
We base the proof on the Theorem 5.6 in \cite{white2011linking}. Firstly, by definition, the system Eq. (\ref{eq:response_function}) belongs to the canonical settable system (Def. 3.3 in \cite{white2011linking}), on which their Theorem 5.6 is based. To prove that in our system Granger causality can deduce direct structural causality, we only have to prove that the assumption A.1 and assumption A.2 in \cite{white2011linking} are satisfied by our system. If we identify our $x_t^{(i)}$ with their $Y_{1,t}$, our $\mathbf{X}_{t-1}$ with their $\mathbf{Y}_{t-1}$, our $x_t^{(j)}$ with their $Y_{2,t}$, our $u_{i,t}$ (our $u_i$ at time $t$) with their $U_{1,t}$, our $u_{j,t}$ with their $U_{2,t}$, their $\mathbf{Z}_t=\emptyset$, $\mathbf{W}_t=\emptyset$, then our system Eq. (\ref{eq:response_function}) satisfies their Assumption A.1. Additionally, by definition, our $u_i\in R^M, i=1,2,...N$ are random variables that are mutually independent, and also independent of any $X^{(i)}_{t-1}, x^{(i)}_t$, $i\in\{1,2,...N\}$. Therefore, our system satisfies their strict exogeneity $(\mathbf{Y}_{t-1},\mathbf{Z}_t)\independent U_{1,t}$ (in our representation  $(\mathbf{X}_{t-1},\emptyset)\independent u_{i,t}$), which is a sufficient condition for Assumption A.2. Therefore, both their Assumption A.1 and Assumption A.2 are satisfied by our system Eq. (\ref{eq:response_function}). Applying their Theorem 5.6, we prove Lemma \ref{thm:theorem_0}.

\end{proof}

Therefore, for our system Eq.~(\ref{eq:response_function}), applying the results by \cite{white2011linking}, we have that Granger causality is a sufficient condition for direct structural causality. The reason that here Granger causality can deduce direct structural causality is in part due to the fact that for system Eq.~(\ref{eq:response_function}), conditional exogeneity \cite{white2011linking} is automatically satisfied. 

Note that the reverse of the statement is not true, i.e. a failed Granger causality test does not necessarily imply that there is no direct structural causality (White \& Lu \cite{white2010granger} give several examples, and also note that these instances are exceptional). 

After stating Assumption \ref{assump:additonal_assumption} and clarifying the definition of causalities, now we prove two lemmas, which are important for the analysis of our objective.

\subsubsection{Minimum MSE with different variables}
\label{app:lemma_2}

\begin{lemma}
\label{lemma:d_separation_mmse}
Suppose that Assumption 1 and \ref{assump:additonal_assumption} holds, and $X_{t-1}^{(U)}$,$X_{t-1}^{(V)}$, $X_{t-1}^{(W)}\subset \X_{t-1}$ are mutually exclusive sets of variables satisfying
\begin{equation*}
X_{t-1}^{(W)}\independent{} x_t^{(i)}|X_{t-1}^{(U)},X_{t-1}^{(V)},\ \ \ \ \ \ \ 
X_{t-1}^{(V)}\not\!\perp\!\!\!\perp x_t^{(i)}|X_{t-1}^{(U)},X_{t-1}^{(W)}
\end{equation*}
Then
\begin{equation*}
\text{min}_{f_\theta}\mathbb{E}_{X_{t-1}^{(U)},X_{t-1}^{(V)},x_t^{(i)}}\left[\left(x_t^{(i)}-f_\theta(X_{t-1}^{(U)},X_{t-1}^{(V)})\right)^2\right]<\text{min}_{f_\theta}\mathbb{E}_{X_{t-1}^{(U)},X_{t-1}^{(V)},x_t^{(i)}}\left[\left(x_t^{(i)}-f_\theta(X_{t-1}^{(U)},X_{t-1}^{(W)})\right)^2\right]
\end{equation*}
\end{lemma}

Fig. S\ref{fig:d_seperation_diagram} below shows the relations between the variables, where the dashed arrows denote the potential existence of causal relations between variables. We see that \emph{conditioned} on $(X_{t-1}^{(U)},X_{t-1}^{(V)})$, we have $x_t^{(i)}$ and $X_{t-1}^{(W)}$ are independent, while \emph{conditioned} on $(X_{t-1}^{(U)},X_{t-1}^{(W)})$, we have $x_t^{(i)}$ and $X_{t-1}^{(V)}$ are \emph{not} independent. Lemma \ref{lemma:d_separation_mmse} states that under the above scenario and under Assumptions \ref{assump:f_theta} and \ref{assump:additonal_assumption}, using $X_{t-1}^{(U)}$ and $X_{t-1}^{(V)}$ to predict $x_t^{(i)}$ can achieve a lower MSE than using $X_{t-1}^{(U)}$ and $X_{t-1}^{(W)}$ to predict $x_t^{(i)}$.

\begin{suppfigure}[h!]
    \centering
    \includegraphics[scale=0.45]{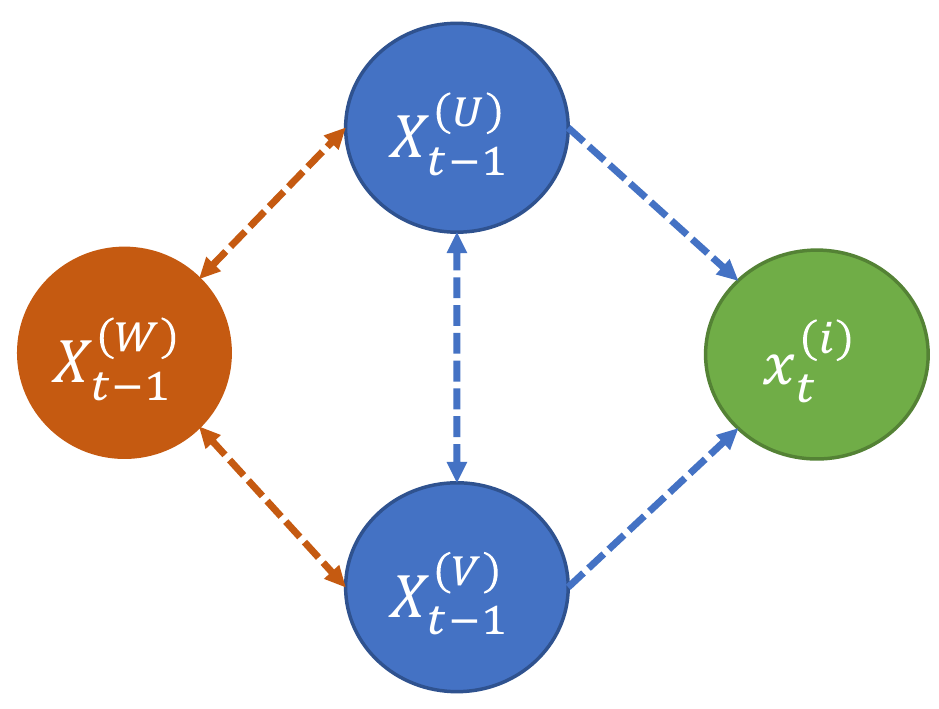}
    \caption{Diagram of variables for Lemma \ref{lemma:d_separation_mmse}. The dashed arrows denote the possible existence of causal relations between variables.}
    \label{fig:d_seperation_diagram}%
\end{suppfigure}

\begin{proof}
Since Assumption \ref{assump:f_theta} holds, according to Lemma \ref{lemma:argmin_risk}, Lemma \ref{lemma:d_separation_mmse} is equivalent to

\begin{equation*}
\begin{aligned}
\mathbb{E}_{X^{(U)}_{t-1},X^{(V)}_{t-1},x_t^{(i)}}&\left[\left(x_t^{(i)}-\int dx^{(i)}_t P(x_t^{(i)}|X^{(U)}_{t-1},X^{(V)}_{t-1})x_t^{(i)}\right)^2\right]\\
&<\mathbb{E}_{X^{(U)}_{t-1},X^{(W)}_{t-1},x_t^{(i)}}\left[\left(x_t^{(i)}-\int dx^{(i)}_t P(x_t^{(i)}|X^{(U)}_{t-1},X^{(W)}_{t-1})x_t^{(i)}\right)^2\right]
\end{aligned}
\end{equation*}
We have
\begin{equation*}
\begin{aligned}
&\mathbb{E}_{X^{(U)}_{t-1},X^{(W)}_{t-1},x_t^{(i)}}\left[\left(x_t^{(i)}-\int dx^{(i)}_t P(x_t^{(i)}|X^{(U)}_{t-1},X^{(W)}_{t-1})x_t^{(i)}\right)^2\right]\\
&=\int dX^{(U)}_{t-1}dX^{(W)}_{t-1}dx_t^{(i)} P(X^{(U)}_{t-1},X^{(W)}_{t-1},x_t^{(i)})\left(x_t^{(i)}-\int dx^{(i)}_t P(x_t^{(i)}|X^{(U)}_{t-1},X^{(W)}_{t-1})x_t^{(i)}\right)^2\\
&=\int dX^{(U)}_{t-1}dX^{(V)}_{t-1}dX^{(W)}_{t-1}dx_t^{(i)} P(X^{(U)}_{t-1},X^{(V)}_{t-1},X^{(W)}_{t-1},x_t^{(i)})\left(x_t^{(i)}-\int dx^{(i)}_t P(x_t^{(i)}|X^{(U)}_{t-1},X^{(W)}_{t-1})x_t^{(i)}\right)^2\\
&=\int dX^{(U)}_{t-1}dX^{(V)}_{t-1}dX^{(W)}_{t-1} P(X^{(U)}_{t-1},X^{(V)}_{t-1})P(X^{(W)}_{t-1}|X^{(U)}_{t-1},X^{(V)}_{t-1})\cdot\\
&\ \ \ \ \ \ \ \ \ \ \ \ \ \ \ \ \ \ \ \ \ \ \ \ \int dx_t^{(i)} P(x_t^{(i)}|X^{(U)}_{t-1},X^{(V)}_{t-1})\left(x_t^{(i)}-\int dx^{(i)}_t P(x_t^{(i)}|X^{(U)}_{t-1},X^{(W)}_{t-1})x_t^{(i)}\right)^2\\
&>\int dX^{(U)}_{t-1}dX^{(V)}_{t-1}dX^{(W)}_{t-1}P(X^{(U)}_{t-1},X^{(V)}_{t-1})P(X^{(W)}_{t-1}|X^{(U)}_{t-1},X^{(V)}_{t-1})\cdot\\
&\ \ \ \ \ \ \ \ \ \ \ \ \ \ \ \ \ \ \ \ \ \ \ \ \int dx_t^{(i)} P(x_t^{(i)}|X^{(U)}_{t-1},X^{(V)}_{t-1})\left(x_t^{(i)}-\int dx^{(i)}_t P(x_t^{(i)}|X^{(U)}_{t-1},X^{(V)}_{t-1})x_t^{(i)}\right)^2\\
&=\int dX^{(U)}_{t-1}dX^{(V)}_{t-1}P(X^{(U)}_{t-1},X^{(V)}_{t-1})\int dx_t^{(i)} P(x_t^{(i)}|X^{(U)}_{t-1},X^{(V)}_{t-1})\left(x_t^{(i)}-\int dx^{(i)}_t P(x_t^{(i)}|X^{(U)}_{t-1},X^{(V)}_{t-1})x_t^{(i)}\right)^2\\
&=\mathbb{E}_{X^{(U)}_{t-1},X^{(V)}_{t-1},x_t^{(i)}}\left[\left(x_t^{(i)}-\int dx^{(i)}_t P(x_t^{(i)}|X^{(U)}_{t-1},X^{(V)}_{t-1})x_t^{(i)}\right)^2\right]
\end{aligned}
\end{equation*}
The third equality (the one before the inequality) is due to that $X_{t-1}^{(W)}\independent{} x_t^{(i)}|X_{t-1}^{(U)},X_{t-1}^{(V)}$, leading to $P(X^{(U)}_{t-1},X^{(V)}_{t-1},X^{(W)}_{t-1},x^{(i)}_{t})=P(X^{(U)}_{t-1},X^{(V)}_{t-1})P(X^{(W)}_{t-1}|X^{(U)}_{t-1},X^{(V)}_{t-1})P(x^{(i)}_{t}|X^{(U)}_{t-1},X^{(V)}_{t-1})$. The inequality step first uses the Assumption \ref{assump:additonal_assumption} that the noise variables $u_i$ are effective arguments of the response functions $h_i$, and that each $h_i$ is ``causality in mean". Therefore, $\int dx^{(i)}_t P(x_t^{(i)}|X^{(U)}_{t-1},X^{(V)}_{t-1})x_t^{(i)}\neq \int dx^{(i)}_t P(x_t^{(i)}|X^{(U)}_{t-1},X^{(W)}_{t-1})x_t^{(i)}$. Using Lemma \ref{lemma:argmin_risk}, we have $f_\theta(X^{(U)}_{t-1},X^{(V)}_{t-1})=\int dx^{(i)}_t P(x_t^{(i)}|X^{(U)}_{t-1},X^{(V)}_{t-1})x_t^{(i)}$ minimizes $\int dx^{(i)}_t P(x_t^{(i)}|X^{(U)}_{t-1},X^{(V)}_{t-1})\left(x_t^{(i)}-f_\theta(X^{(U)}_{t-1},X^{(V)}_{t-1})\right)^2$, hence the inequality.
\end{proof}

Using Lemma \ref{lemma:d_separation_mmse} recursively, we see that using variables that directly causes $x^{(i)}$ to predict $x^{(i)}$ can achieve the lowest MSE. Formalizing the above intuition, we have

\begin{lemma}
\label{thm:d_separation_mmse}
Suppose that Assumption 1 and \ref{assump:additonal_assumption} holds, and $X_{t-1}^{(D)}\subseteq\X_{t-1}$ are the set of variables that directly causes $x_t^{(i)}$. Then $\forall X_{t-1}^{(S)}\subseteq\X_{t-1}$ with $X_{t-1}^{(S)}\neq X_{t-1}^{(D)}$, we have

\begin{equation*}
\text{min}_{f_\theta}\mathbb{E}_{X_{t-1}^{(D)}}\left[\left(x_t^{(i)}-f_\theta(X_{t-1}^{(D)})\right)^2\right]<\text{min}_{f_\theta}\mathbb{E}_{X_{t-1}^{(S)}}\left[\left(x_t^{(i)}-f_\theta(X_{t-1}^{(S)})\right)^2\right]
\end{equation*}

Specifically, we have
\begin{equation*}
\text{min}_{f_\theta}\mathbb{E}_{X_{t-1}^{(D)}}\left[\left(x_t^{(i)}-f_\theta(X_{t-1}^{(D)})\right)^2\right]<\text{min}_{f_\theta}\mathbb{E}_{X_{t-1}^{(\hat{D})}}\left[\left(x_t^{(i)}-f_\theta(X_{t-1}^{(\hat{D})})\right)^2\right]
\end{equation*}

where $X_{t-1}^{(\hat{D})}=\X_{t-1}\textbackslash X_{t-1}^{(D)}$.
\end{lemma}

\begin{proof}
For any $X^{(S)}_{t-1}$, let $X^{(U)}_{t-1}=X^{(D)}_{t-1}\cap X^{(S)}_{t-1}$, $X^{(V)}_{t-1}=X^{(D)}_{t-1}\textbackslash X^{(S)}_{t-1}$, $X^{(W)}_{t-1}=X^{(S)}_{t-1}\textbackslash X^{(D)}_{t-1}$. Then $X^{(U)}_{t-1}$, $X^{(V)}_{t-1}$, $X^{(W)}_{t-1}$ are mutually exclusive, and $X^{(D)}_{t-1}=X^{(U)}_{t-1}\cup X^{(V)}_{t-1}$, $X^{(S)}_{t-1}=X^{(U)}_{t-1}\cup X^{(W)}_{t-1}$. Now we prove that $\forall X_{t-1}^{(S)}\subseteq\X_{t-1}$ with $X_{t-1}^{(S)}\neq X_{t-1}^{(D)}$, the corresponding $X^{(U)}_{t-1}$, $X^{(V)}_{t-1}$, $X^{(W)}_{t-1}$, $x_t^{(i)}$ satisfy the condition for Lemma \ref{lemma:d_separation_mmse}.
Since $X^{(D)}_{t-1}$ are the set of variables that directly causes $x_t^{(i)}$, there does not exist a $X^{(S)}_{t-1}$ such that the corresponding $X_{t-1}^{(V)}\independent{} x_t^{(i)}|X_{t-1}^{(U)},X_{t-1}^{(W)}$ (otherwise it violates the direct causality). Thus $X_{t-1}^{(V)}\not\!\perp\!\!\!\perp x_t^{(i)}|X_{t-1}^{(U)},X_{t-1}^{(W)}$. To prove $X_{t-1}^{(W)}\independent{} x_t^{(i)}|X_{t-1}^{(U)},X_{t-1}^{(V)}$, note that $X^{(W)}_{t-1}$ does not directly cause $x_t^{(i)}$, then  $X^{(W)}_{t-1}$ does not Granger-cause $x_t^{(i)}$, i.e. $P(x_t^{(i)}|X_{t-1}^{(U)},X_{t-1}^{(V)})=P(x_t^{(i)}|X_{t-1}^{(U)},X_{t-1}^{(V)},X_{t-1}^{(W)})$, which is equivalent to $X_{t-1}^{(W)}\independent{} x_t^{(i)}|X_{t-1}^{(U)},X_{t-1}^{(V)}$. The special case of $X_{t-1}^{(\hat{D})}$ follows directly that $X_{t-1}^{(\hat{D})}=\X_{t-1}\textbackslash X_{t-1}^{(D)}\neq X_{t-1}^{(D)}$ and letting $X_{t-1}^{(S)}=X_{t-1}^{(\hat{D})}$.
\end{proof}

\subsubsection{Qualitative and quantitative behaviors of the mutual information-regularized risk}
\label{app:analysis_risk}

In this section, we analyze the qualitative and quantitative behaviors of the mutual information-regularized risk (Eq. \ref{eq:learnable_risk}), with varying noise levels $\eta_j$. For each variable $X_{t-1}^{(j)}\in\X_{t-1}$, $j=1,2,...N$, define $\rho_{j}=\text{tanh}\left(I(X_{t-1}^{(j)};\tilde{X}_{t-1}^{(j)(\eta_j)})\right)\in[0,1]$ as a ``rescaled" mutual information between $X_{t-1}^{(j)}$ and $\tilde{X}_{t-1}^{(j)(\eta_j)}$. When $\eta_j=\mathbf{0}$ so that $\tilde{X}_{t-1}^{(j)(\eta_j)}=X_{t-1}^{(j)}$, $\rho_j=1$, at which the input $X_{t-1}^{(j)}$ is fully preserved. When all elements of $\eta_j\to\infty$, $\rho_j=0$, at which $X_{t-1}^{(j)}$ is fully corrupted. Denoting $\boldsymbol{\rho}=(\rho_1,\rho_2,...\rho_N)$, we can then rewrite the mutual information-regularized risk (Eq. \ref{eq:learnable_risk}) as

\begin{equation}
\label{eq:learnable_risk_rho}
R_{\mathbf{X},x^{(i)}}[f_\theta,\boldsymbol{\rho}]=\text{MMSE}^{(i)}(\boldsymbol{\rho})+\lambda\cdot \sum_{j=1}^{N}\text{arctanh}(\rho_{j})
\end{equation}

where $\text{MMSE}^{(i)}(\boldsymbol{\rho})=\min_{\boldsymbol{\eta},f_\theta}\mathbb{E}_{\mathbf{X}_{t-1},x_t^{(i)},\boldsymbol{\epsilon}}\left[\left(x_t^{(i)}-f_\theta(\tilde{\mathbf{X}}^{(\boldsymbol{\eta})}_{t-1})\right)^2\right]$ subject to $\rho_{j}=\text{tanh}\left(I(X_{t-1}^{(j)};\tilde{X}_{t-1}^{(j)(\eta_j)}\right), j=1,2,...N$. Let $X_{t-1}^{(D)}\subseteq\X_{t-1}$ be the set of variables that directly causes $x_t^{(i)}$, and denote the corresponding set of $\rho_j$ as $\boldsymbol{\rho}^{(D)}$. Denote $X_{t-1}^{(\hat{D})}=\X_{t-1}\textbackslash X_{t-1}^{(D)}$ and the corresponding set of $\rho_j$ as $\boldsymbol{\rho}^{(\hat{D})}$. For any $i=1,2,...N$, it is easy to see that $\text{MMSE}^{(i)}(\boldsymbol{\rho})$ has the following properties:

\begin{enumerate}
\item $\text{MMSE}^{(i)}(\boldsymbol{\rho})$ attains maximum at $\boldsymbol{\rho}=\mathbf{0}$.
\item $\text{MMSE}^{(i)}(\boldsymbol{\rho})$ is monotonically decreasing w.r.t. each $\rho_j$.
\item $\text{MMSE}^{(i)}(\boldsymbol{\rho})\big|_{\boldsymbol{\rho}^{(D)}=\mathbf{1},\boldsymbol{\rho}^{(\hat{D})}=\mathbf{0}}<\text{MMSE}^{(i)}(\boldsymbol{\rho})\big|_{\boldsymbol{\rho}^{(D)}=\mathbf{0},\boldsymbol{\rho}^{(\hat{D})}=\mathbf{1}}$ (using Lemma \ref{thm:d_separation_mmse}).
\item $\text{MMSE}^{(i)}(\boldsymbol{\rho})$ attains minimum at $\boldsymbol{\rho}^{(D)}=\mathbf{1}$. $\text{MMSE}^{(i)}(\boldsymbol{\rho})\big|_{\boldsymbol{\rho}^{(D)}=\mathbf{1}}$ is constant w.r.t. $\boldsymbol{\rho}^{(\hat{D})}$.
\end{enumerate}

To get a better intuition of the landscape of $R_{\mathbf{X},x^{(i)}}[f_\theta,\boldsymbol{\rho}]$, let's investigate a simple example. Let the response function be:
\begin{equation}
\label{eq:response_function_app}
    \begin{cases}
      x^{(1)}_t:=h_1(u_1)=\sqrt{\Sigma_x}\cdot u_1 \\
      x^{(2)}_t:=h_2(x^{(1)}_{t-1},u_2)=x^{(1)}_{t-1}+\sqrt{\Omega_x}\cdot u_2 \\
      x^{(3)}_t:=h_3(x^{(2)}_{t-1},u_3)=x^{(2)}_{t-1}+\sqrt{\Omega_y}\cdot u_3
    \end{cases}
  \end{equation}
where $u_1,u_2,u_3$ are independent unit Gaussian variables, and $\X_{t-1}=(X^{(1)}_{t-1},X^{(2)}_{t-1},X^{(3)}_{t-1})=\left((x_{t-2}^{(1)},x_{t-1}^{(1)}),(x_{t-2}^{(2)},x_{t-1}^{(2)}),(x_{t-2}^{(3)},x_{t-1}^{(3)})\right)$. For $R_{\mathbf{X},x^{(3)}}[f_\theta,\boldsymbol{\rho}]=\text{MMSE}^{(3)}(\boldsymbol{\rho})+\lambda\cdot \sum_{j=1}^{3}\text{arctanh}(\rho_{j})$, since only $x_{t-2}^{(1)}$ and $x_{t-1}^{(2)}$ are d-connected to $x_t^{(3)}$, at the minimization of $R_{\mathbf{X},x^{(3)}}[f_\theta,\boldsymbol{\rho}]$, only $x_{t-2}^{(1)}$ and $x_{t-1}^{(2)}$ may have a finite $\eta_{j,l}^*$ (the other $\eta_{j,l}^*$ are all infinite). Therefore, setting the $\eta_{j,l}$ not corresponding to $x_{t-2}^{(1)}$ and $x_{t-1}^{(2)}$ as infinity, and let $\tilde{x}_{t-2}^{(1)}=x_{t-2}^{(1)}+\eta_x\cdot\epsilon_x$, $\tilde{x}_{t-1}^{(2)}=x_{t-1}^{(2)}+\eta_y\cdot\epsilon_y$, $\epsilon_x$ and $\epsilon_y$ being independent unit Gaussian variables. Let $f_\theta (x_{t-2}^{(1)},x_{t-1}^{(2)})=a\cdot x_{t-2}^{(1)} + b\cdot x_{t-1}^{(2)}$, then we can get an analytic expression for $R_{\mathbf{X},x^{(3)}}[f_\theta,\eta_x,\eta_y]$:

\begin{equation*}
\begin{aligned}
&R_{\mathbf{X},x^{(3)}}[f_\theta,\eta_x,\eta_y]\\
&=a^2\Sigma_x+(b-1)^2(\Sigma_x+\Omega_x)+a^2\eta_x^2+b^2\eta_y^2+2a(b-1)\Sigma_x+\Omega_y+\frac{\lambda}{2}\text{log}\left(1+\frac{\Sigma_x}{\eta_x^2}\right)+\frac{\lambda}{2}\text{log}\left(1+\frac{\Sigma_x+\Omega_x}{\eta_y^2}\right)
\end{aligned}
\end{equation*}

Minimizing $R_{\mathbf{X},x^{(3)}}[f_\theta,\eta_x,\eta_y]$ w.r.t. $a$ and $b$, we get

\begin{equation*}
\begin{aligned}
a^*&= \frac{\eta_y^2\Sigma_x}{\eta_x^2\eta_y^2+\eta_x^2\Sigma_x+\eta_y^2\Sigma_x+\eta_x^2\Omega_x+\Omega_x\Sigma_x}\\
b^*&= \frac{\eta_x^2(\Sigma_x+\Omega_x)+\Sigma_x\Omega_x}{\eta_x^2\eta_y^2+\eta_x^2\Sigma_x+\eta_y^2\Sigma_x+\eta_x^2\Omega_x+\Omega_x\Sigma_x}
\end{aligned}
\end{equation*}

Substituting into $R_{\mathbf{X},x^{(3)}}[f_\theta,\eta_x,\eta_y]$, we have
\begin{equation*}
\begin{aligned}
&R_{\mathbf{X},x^{(3)}}[\eta_x,\eta_y]\\
&=\min_{f_\theta}R_{\mathbf{X},x^{(3)}}[f_\theta,\eta_x,\eta_y]\\
&=\frac{\eta_y^2(\Sigma_x\Omega_x+\eta_x^2(\Sigma_x+\Omega_x))}{\eta_x^2\eta_y^2+\eta_x^2\Sigma_x+\eta_y^2\Sigma_x+\eta_x^2\Omega_x+\Omega_x\Sigma_x}+\frac{\lambda}{2}\text{log}\left(1+\frac{\Sigma_x}{\eta_x^2}\right)+\frac{\lambda}{2}\text{log}\left(1+\frac{\Sigma_x+\Omega_x}{\eta_y^2}\right)
\end{aligned}
\end{equation*}

Here we have neglected the constant $\Omega_y$. To obtain $R_{\mathbf{X},x^{(3)}}[\boldsymbol{\rho}]$, let $\rho_1=\text{tanh}\left(\frac{1}{2}\text{log}\left(1+\frac{\Sigma_x}{\eta_x^2}\right)\right)$, $\rho_2=\text{tanh}\left(\frac{1}{2}\text{log}\left(1+\frac{\Sigma_x+\Omega_x}{\eta_x^2}\right)\right)$, we have $\eta_x^2=\frac{1-\rho_1}{2\rho_1}\Sigma_x$, $\eta_y^2=\frac{1-\rho_2}{2\rho_2}(\Sigma_x+\Omega_x)$. Substituting, we have

\begin{equation*}
\begin{aligned}
&R_{\mathbf{X},x^{(3)}}[\boldsymbol{\rho}]=\text{MMSE}^{(3)}(\boldsymbol{\rho})+\lambda\cdot \sum_{j=1}^{2}\text{arctanh}(\rho_{j})\\
&=\frac{(\rho_2-1)(\Sigma_x+\Omega_x)((\rho_1-1)\Sigma_x-(\rho_1+1)\Omega_x)}{(1+\rho_1+\rho_2-3\rho_1\rho_2)\Sigma_x+(1+\rho_1)(1+\rho_2)\Omega_x}+\lambda\cdot \text{arctanh}(\rho_1)+\lambda\cdot\text{arctanh}(\rho_2)
\end{aligned}
\end{equation*}

Fig. S\ref{fig:risk} shows the landscape of $\text{MMSE}^{(3)}(\boldsymbol{\rho})$ and $R_{\mathbf{X},x^{(3)}}[\boldsymbol{\rho}]$, for $\Sigma_x=1,\Omega_x=2,\lambda=1$. We see that $\text{MMSE}^{(3)}(\boldsymbol{\rho})$ satisfies the above mentioned four properties. Particularly, $\text{MMSE}^{(3)}(\boldsymbol{\rho})\big|_{\rho_1=1,\rho_2=0}>\text{MMSE}^{(3)}(\boldsymbol{\rho})\big|_{\rho_1=0,\rho_2=1}$. After adding $\lambda\cdot \text{arctanh}(\rho_1)+\lambda\cdot\text{arctanh}(\rho_2)$, the $R_{\mathbf{X},x^{(3)}}[\boldsymbol{\rho}]$ has global minimum along $\rho_1=0$ largely due to this property. Therefore, for this particular example, when $R_{\mathbf{X},x^{(3)}}[\boldsymbol{\rho}]$ is minimized, $\rho_1=0$, i.e. $I(x_{t-2}^{(1)},\tilde{x}_{t-2}^{(1)(\eta_1^*)})=0$.

\begin{suppfigure}
    \centering
    \begin{subfigure}{.45\linewidth}
    \includegraphics[scale=0.45]{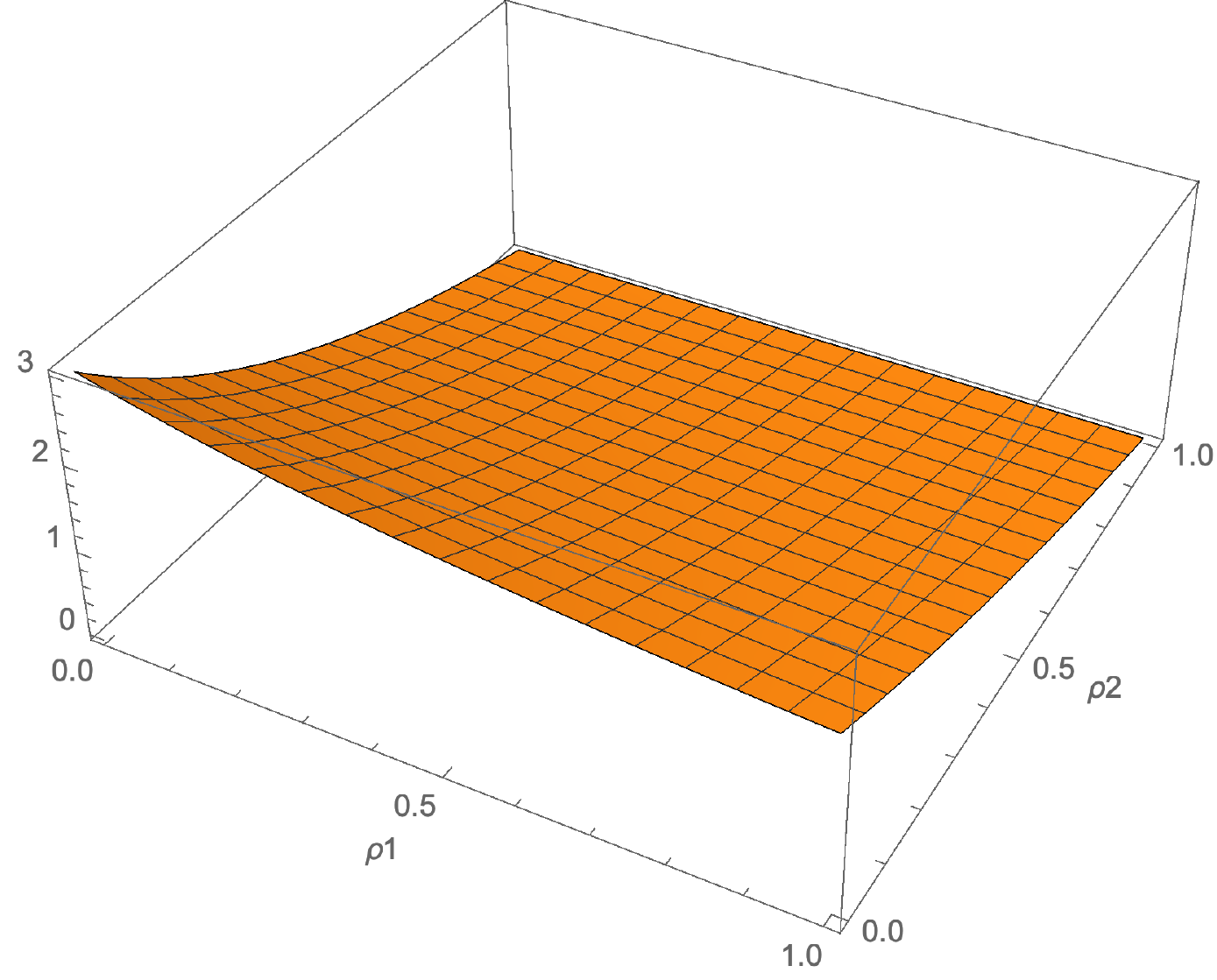}
    \caption{}
    \end{subfigure}
    \begin{subfigure}{.45\linewidth}
    \includegraphics[scale=0.45]{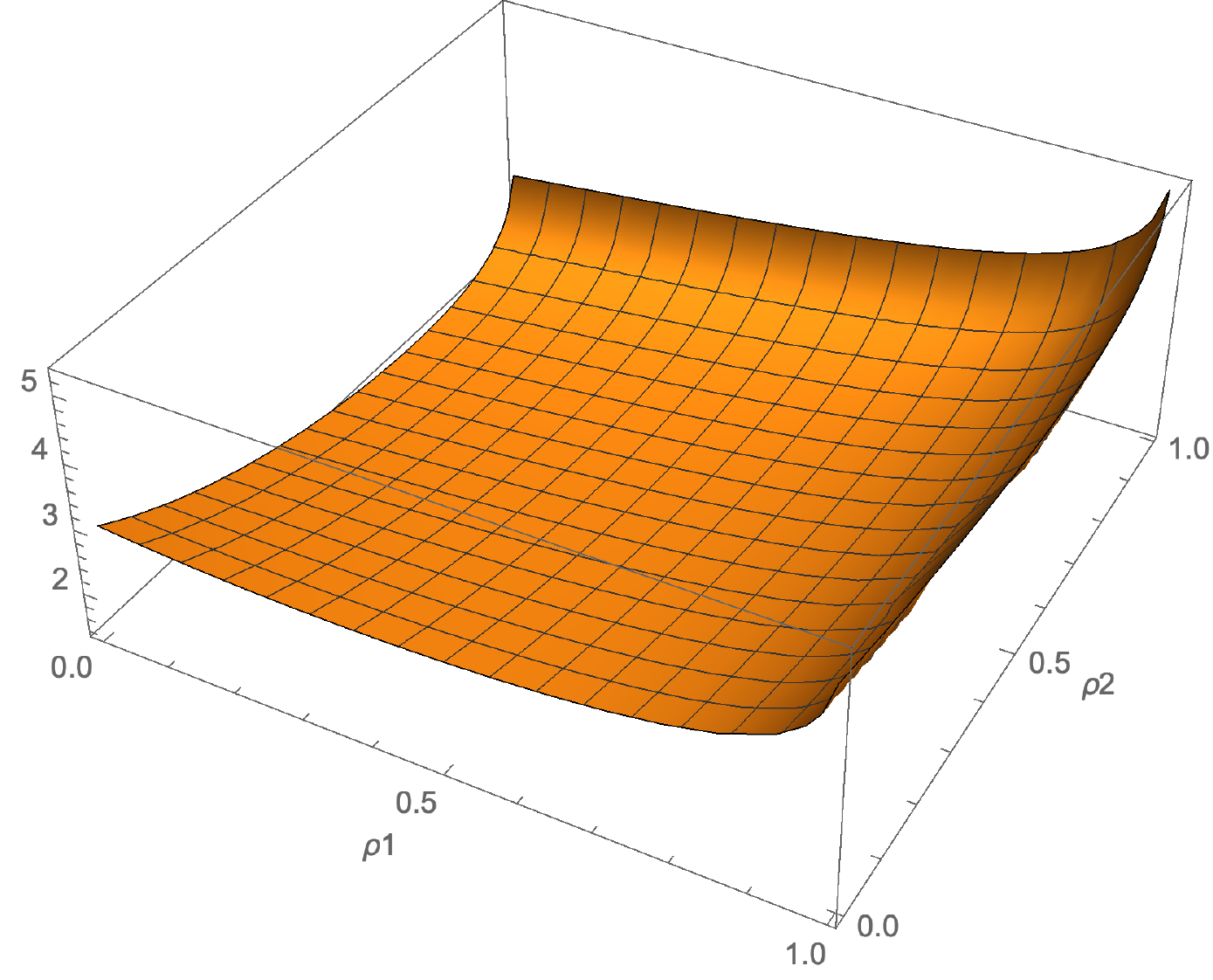}
    \caption{}
    \end{subfigure}
    \caption{(a) $\text{MMSE}^{(3)}(\boldsymbol{\rho})$ and (b) $R_{\mathbf{X},x^{(3)}}[\boldsymbol{\rho}]$ in section \ref{app:analysis_risk}, for $\Sigma_x=1,\Omega_x=2,\lambda=1$.}%
    \label{fig:risk}%
\end{suppfigure}

By varying the value of $\lambda$, we can tune the relative influence of the two terms $\text{MMSE}^{(3)}(\boldsymbol{\rho})$ and $\sum_{j=1}^{2}\text{arctanh}(\rho_{j})$. The landscape corresponding to $\lambda=0.01,0.5,2,10$ are plotted in Fig. S\ref{fig:risk_tune}. We see that when $\lambda\ll1$, the MMSE term dominates, and it is possible that the global minimum of $R_{\mathbf{X},x^{(3)}}[\boldsymbol{\rho}]$ is not at $\rho_1=0$. This is similar to the effect of a L1 regularization, where if the coefficient $\lambda$ for the L1 is vanishingly small, the L1 regularization will barely influence the loss landscape. When $\lambda$ is not vanishingly small, as in Fig. S\ref{fig:risk_tune} (b), we see that the global minimum of $R_{\mathbf{X},x^{(3)}}[\boldsymbol{\rho}]$ lies on $\rho_1=0$. When $\lambda\to+\infty$, the $\sum_{j=1}^{2}\text{arctanh}(\rho_{j})$ term dominates and the global minimum is at $\rho_1=0,\rho_2=0$.

In general, we expect $R_{\mathbf{X},x^{(i)}}[\boldsymbol{\rho}]$ behave qualitatively similar. When $\lambda\to+\infty$, the global minimum for $R_{\mathbf{X},x^{(i)}}[\boldsymbol{\rho}]$ is at $\boldsymbol{\rho}^*=\mathbf{0}$. As we ramp down $\lambda$, the dimension that has largest influence on MMSE will first host the global minimum with nonzero $\rho_j^*$, which is most likely the variable that directly causes $x_i^{(i)}$. When $\lambda$ is further ramping down, we expect that the variables that host the global minimum with nonzero $\rho_j$ will more likely be those that directly causes $x_i^{(i)}$, due to the landscape influenced by the four properties of MMSE. This can justify the mutual information-regularized risk as a good objective for causal discovery/variable selection. The experiments in the paper will empirically test the performance of the mutual information-regularized risk.

\begin{suppfigure}
    \centering
    \begin{subfigure}{.45\linewidth}
    \includegraphics[scale=0.45]{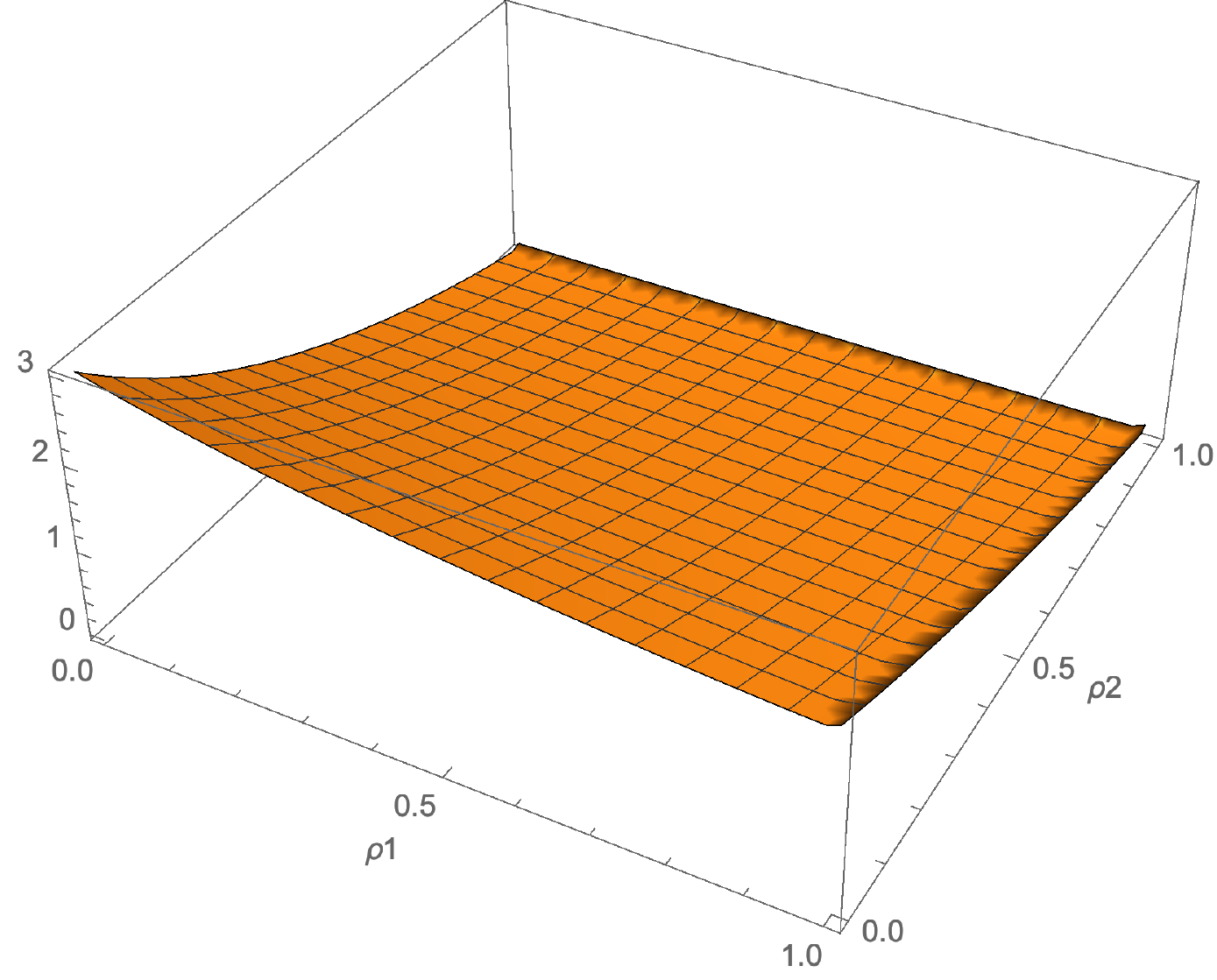}
    \caption{}
    \end{subfigure}
    \begin{subfigure}{.45\linewidth}
    \includegraphics[scale=0.45]{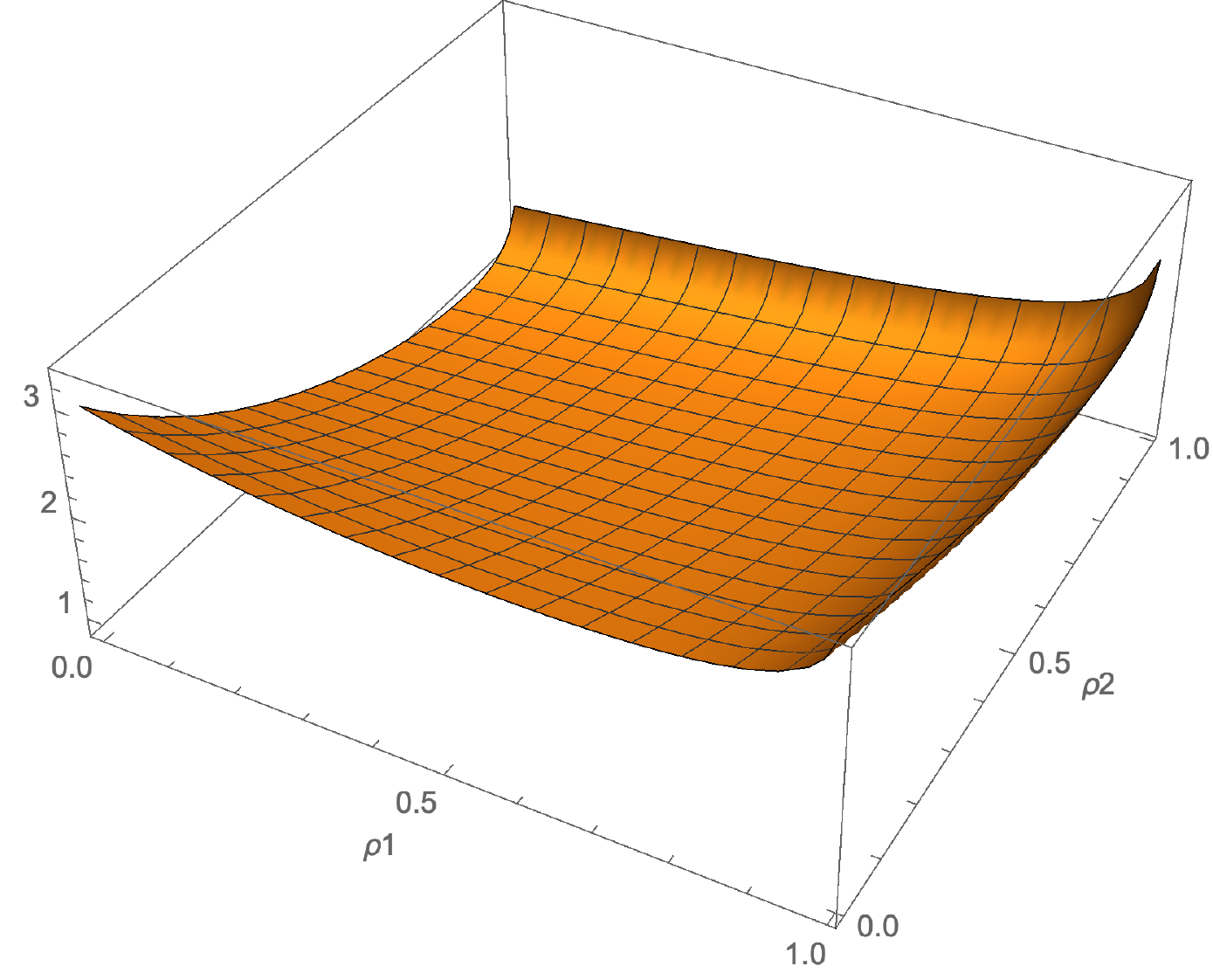}
    \caption{}
    \end{subfigure}
    \begin{subfigure}{.45\linewidth}
    \includegraphics[scale=0.45]{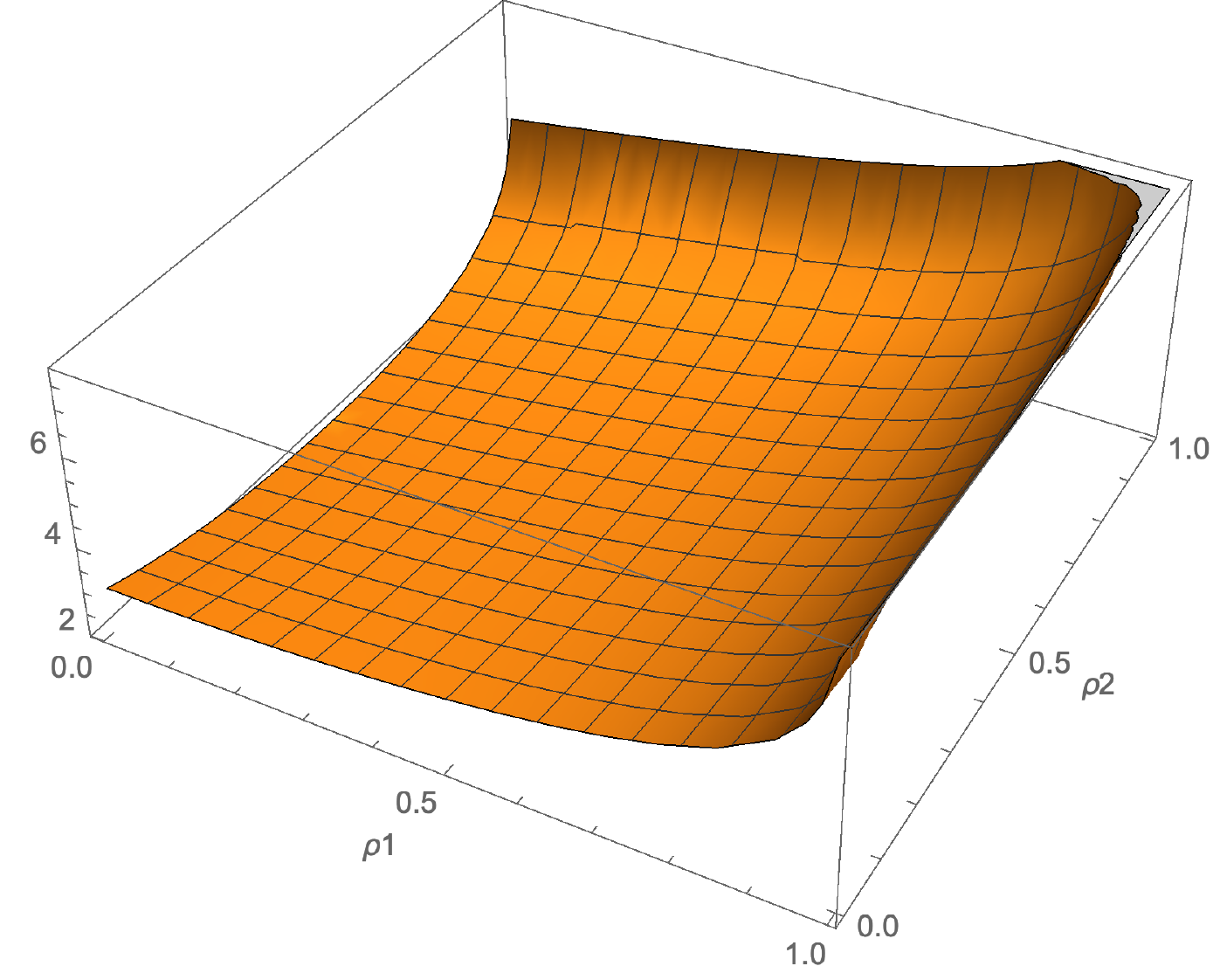}
    \caption{}
    \end{subfigure}
    \begin{subfigure}{.45\linewidth}
    \includegraphics[scale=0.45]{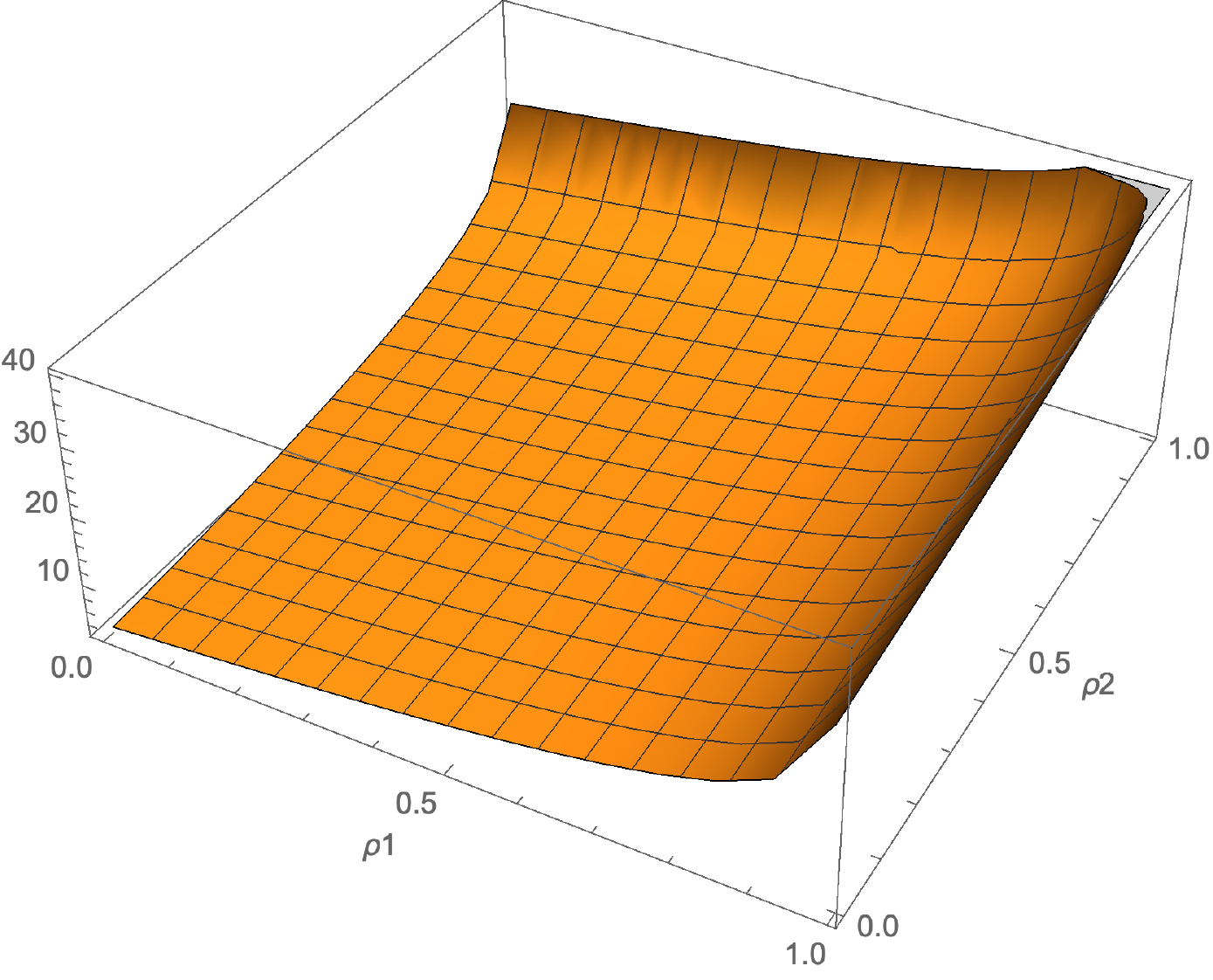}
    \caption{}
    \end{subfigure}
    \caption{(a) $R_{\mathbf{X},x^{(3)}}[\boldsymbol{\rho}]$ for (a) $\lambda=0.01$, (b) $\lambda=0.5$, (c) $\lambda=2$ and (d) $\lambda=10$ in section \ref{app:analysis_risk}, for $\Sigma_x=1,\Omega_x=2$.}%
    \label{fig:risk_tune}%
\end{suppfigure}

\section{Upper bound for the mutual information-regularized risk}
\label{app:Gaussian_channel_upper_bound}

In this section, we prove that $I(\tilde{X}^{(j)(\eta_j)}_{t-1};X^{(j)}_{t-1})\leq \frac{1}{2}\sum_{l=1}^{KM}\text{log}\left(1+\frac{\text{Var}(X_{t-1,l}^{(j)})}{\eta_{j,l}^2}\right)$. We formally state the theorem as follows:

\begin{theorem}
Let $\tilde{X}_{t-1}^{(j)(\eta_j)}:=X_{t-1}^{(j)}+\eta_j\cdot \epsilon_j$, $j=1,2,...N$ be the noise-corrupted inputs with \emph{learnable} noise amplitudes $\eta_j\in \R^{KM}$, and $\epsilon_j\sim N(\mathbf{0},\mathbf{I})$. We have

\begin{equation}
I(\tilde{X}^{(j)(\eta_j)}_{t-1};X^{(j)}_{t-1})\leq \frac{1}{2}\sum_{l=1}^{KM}\text{log}\left(1+\frac{\text{Var}(X_{t-1,l}^{(j)})}{\eta_{j,l}^2}\right)
\end{equation}

where $l$ is the $l^{\text{th}}$ element of a vector, $\text{std}(X_{t-1,l}^{(j)})$ is the standard deviation of $X_{t-1,l}^{(j)}$ across $t$. The equality is reached when $X_{t-1}^{(j)}$ obeys a multivariate Gaussian distribution with diagonal covariance matrix $\Sigma$ satisfying $\Sigma_{l,l}=\text{Var}(X^{(j)}_{t-1,l})+\eta_{j,l}^2$.
\end{theorem}

\begin{proof}
We have
\begin{equation*}
\begin{aligned}
&I(\tilde{X}^{(j)(\eta_j)}_{t-1};X^{(j)}_{t-1})=H(\tilde{X}^{(j)(\eta_j)}_{t-1}) - H(\eta_j\cdot\epsilon_j)\\
&=H(\tilde{X}^{(j)(\eta_j)}_{t-1}) - \left(\frac{KM}{2}\text{log}(2\pi e)+\sum_{l=1}^{KM}\frac{1}{2}\text{log}(\eta_{j,l}^2)\right)\\
\end{aligned}
\end{equation*}

Here $H(\cdot)$ is differential entropy. For $\tilde{X}^{(j)(\eta_j)}_{t-1}$, its variance at the $l^{\text{th}}$ dimension is
\begin{equation*}
\begin{aligned}
\text{Var}(\tilde{X}^{(j)(\eta_{j})}_{t-1,l})&=\text{Var}(X^{(j)}_{t-1,l}+\eta_j\cdot\epsilon_j)\\
&=\text{Var}(X^{(j)}_{t-1,l})+\text{Var}(\eta_{j,l}\cdot\epsilon_{j,l})\\
&=\text{Var}(X^{(j)}_{t-1,l})+\eta_{j,l}^2
\end{aligned}
\end{equation*}

The second equality is due to that $X_{t-1}^{(j)}$ is independent of $\epsilon_j$. Using the principle of maximum entropy, the distribution that maximizes $H(\tilde{X}^{(j)(\eta_j)}_{t-1})$ subject to the constraint of $\text{Var}(\tilde{X}^{(j)(\eta_{j})}_{t-1,l})=\text{Var}(X^{(j)}_{t-1,l})+\eta_{j,l}^2, l=1,2,...KM$ is a Gaussian distribution whose diagonal covariance matrix $\Sigma$ satisfies $\Sigma_{l,l}=\text{Var}(X^{(j)}_{t-1,l})+\eta_{j,l}^2$. Its entropy is $H(\tilde{X}^{(j)(\eta_j)}_{t-1})=\frac{KM}{2}\text{log}(2\pi e)+\sum_{l=1}^{KM}\frac{1}{2}\text{log}(\eta_{j,l}^2+\text{Var}(X^{(j)}_{t-1,l}))$. Therefore,

\begin{equation*}
\begin{aligned}
&I(\tilde{X}^{(j)(\eta_j)}_{t-1};X^{(j)}_{t-1})\\
&\leq \left(\frac{KM}{2}\text{log}(2\pi e)+\sum_{l=1}^{KM}\frac{1}{2}\text{log}(\eta_{j,l}^2+\text{Var}(X^{(j)}_{t-1,l}))\right)-\left(\frac{KM}{2}\text{log}(2\pi e)+\sum_{l=1}^{KM}\frac{1}{2}\text{log}(\eta_{j,l}^2)\right)\\
&=\frac{1}{2}\sum_{l=1}^{KM}\text{log}\left(1+\frac{\text{Var}(X_{t-1,l}^{(j)})}{\eta_{j,l}^2}\right)\\
\end{aligned}
\end{equation*}
The equality is reached when $X_{t-1}^{(j)}$ obeys a multivariate Gaussian distribution with diagonal covariance matrix $\Sigma$ satisfying $\Sigma_{l,l}=\text{Var}(X^{(j)}_{t-1,l})+\eta_{j,l}^2$.
\end{proof}

\section{Implementation details for the methods}
\label{app:algorithm_implementation}

Here we state the implementation details for our method, as well as other methods being compared. Throughout this paper, unless otherwise specified, we use the standard k-nearest neighbor technique in \cite{kraskov2004estimating} to estimate the KL-divergence and mutual information (with number of neighbors $k=5$) and conditional mutual information (with number of neighbors $k=3$), which is used in our implementations of Mutual information, Transfer Entropy and Causal Influence.

\subsection{Our method}

Without stating otherwise, our method (Algorithm \ref{alg:learnable_noise}) as a default uses a three layer neural net, with two hidden layers having 8 neurons and leakyReLU ($\text{max}(0.3x,x)$) activation, and the last layer having linear activation. We set the number of fake time series $S=\text{max}(2,\ceil[\big]{N/2})$, and significance level $\alpha=0.05$. Adam \cite{kingma2014adam} optimizer with learning rate $=10^{-4}$ is used as default throughout this paper. We set $\eta_0=0.01$ and $\lambda=0.002$. We use 30000 epochs.
It also has a 400 epoch warm-up period where the mutual information term is turned off, to allow $f_\theta$ to find a good initial model as a start. We use the the upper bound (Eq. \ref{eq:empirical_upper_bound}) as the risk and also in estimating $W_{ji}$, as discussed in the main text in Section \ref{sec:our_method}. 
In this work, the relative noise amplitude $\chi_{j,l}=\frac{\eta_{j,l}}{\text{std}(X_{t-1,l}^{(j)})}$ is shared across the dimension $l$ for each time series $j$. This simplifies the risk calculation, and is invariant to the rescaling of each time series $X_{t-1}^{(j)}$. We also tested fully parameterizing $\chi_{j,l}$ with a similar performance.

\subsection{Transfer Entropy}
We use the definition of transfer entropy as defined in \cite{schreiber2000measuring}. In that work the transfer entropy is defined for two time series. To deal with multiple time series, we let $X_{t-1}^{(\hat{j})}$ also include other time series, similar to the extension of transfer entropy as in \cite{lizier2008local}.

\subsection{Causal Influence}
For causal influence \cite{janzing2013quantifying}, we use the same network architecture as in our method, to learn a prediction model. Then the KL divergence is estimated via the technique in \cite{kraskov2004estimating}.

\subsection{Linear Granger}

We follow the definition of linear Granger causality (Eq. (7) and (8) in \cite{ding2006granger}) to calculate linear Granger causality. Specifically, we calculate the residual squared error of a linear predictor of $x_{t-1}^{(i)}$ with and without $X_{t-1}^{(j)}$ (both with $\mathbf{X}_{t-1}^{(\hat{j})}$). Then the linear Granger causality equals the log of the ratio of the two residual squared errors.

\subsection{Kernel Granger}

We use the implementation\footnote{At \href{https://github.com/danielemarinazzo/KernelGrangerCausality}{https://github.com/danielemarinazzo/KernelGrangerCausality}.} for \cite{marinazzo2008kernel,marinazzo2008kernel2} for estimating kernel Granger causality. We use their default settings, with inhomogeneous polynomial (IP) kernel of degree $p=2$. We follow the normalization requirement of the algorithm to normalize the data for each experiment.

\subsection{Elastic Net}
We use elastic net \cite{zou2005regularization} with 5-fold time-series-split cross-validation, along the following regularization path: L1-ratio: 0.5, 0.8, 0.9, 0.95, 0.99, and strength of penalization $\alpha$ being a 200-step geometric series from $10^{-4}$ to $10^{-0.5}$. The score function used for cross-validation is the coefficient of determination ($R^2$). The elastic net is implemented with scikit-learn's ElasticNetCV module\footnote{At \href{https://scikit-learn.org/stable/modules/generated/sklearn.linear_model.ElasticNetCV.html}{https://scikit-learn.org/stable/modules/generated/sklearn.linear\_model.ElasticNetCV.html}.}, with optimization tolerance of $10^{-10}$.

\subsection{Gaussian Random}
For Gaussian Random, we draw 10,000 random matrices, each element of which is drawn from a standard Gaussian distribution.

\section{Implementation details for synthetic experiments}
\label{app:synthetic_exp}

For all experiments in this section, each metric is obtained by performing the experiments (including generation of the dataset and the training) ten times with seed = $0, 30, 60, 90, 120,150,180,210,240,270$ and averaging the resulting metrics (for Gaussian random matrices, for each true causal matrix $A$ sample 10,000 random matrices $\tilde{A}$). For the ground-truth causal tensor $A$, each element $A_{ji}$ is a $K\times M$ matrix, with 0.5 probability of being an all-zero matrix, and 0.5 probability of being a nonzero matrix. If $A_{ji}$ is a nonzero matrix, its each element is sampled from a log-normal distribution with $\mu=0$ and $\sigma=1$. For $B$, each $B_j$ is also a $K\times N$ matrix, with each element sampling from $U[-1,1]$. We use $\text{H}_1(x)=\text{softplus}(x)=\text{log}(1+e^x)$, and $\text{H}_2(x)=\text{tanh}(x)$ in equation (\ref{eq:synthetic}). 
As a default, 500 time series each with length of 22 are generated from Eq. (\ref{eq:synthetic}), each of which is wrapped into 19 $(\mathbf{X}_{t-1}, x_{t}^{(i)})$ pairs (since $K=3$), so there are in total $500 \times 19=9500$ examples for each dataset. Since we are using AUC as metrics where a threshold is not necessary, we neglect step 10 in Alg. \ref{alg:learnable_noise} for synthetic experiment. The train-test-split is 9:1 for all experiments in this paper. See Fig. S\ref{fig:synthetic_example_figure} for example snapshots of time series together with the corresponding $A_{ji}$ matrices.

\section{AUC-ROC table for synthetic experiment}
\label{app:synthetic_ROC_AUC}
Table S\ref{table:synthetic_larger_N_AUC_PR} show the AUC-ROC table for the synthetic experiment, where for each $N$, 10 datasets are randomly sampled according to Eq. (\ref{eq:synthetic}) using random seed 0, 30, 60, 90, 120, 150, 180, 210, 240, 270, over which each method is run and their metrics are accumulated. It has similar behavior as the AUC-PR table (Table \ref{table:synthetic_larger_N_AUC_PR}) in the main text.

\begin{suppfigure}[h!]
\centering
\begin{subfigure}{1\columnwidth}
\includegraphics[scale=0.33]{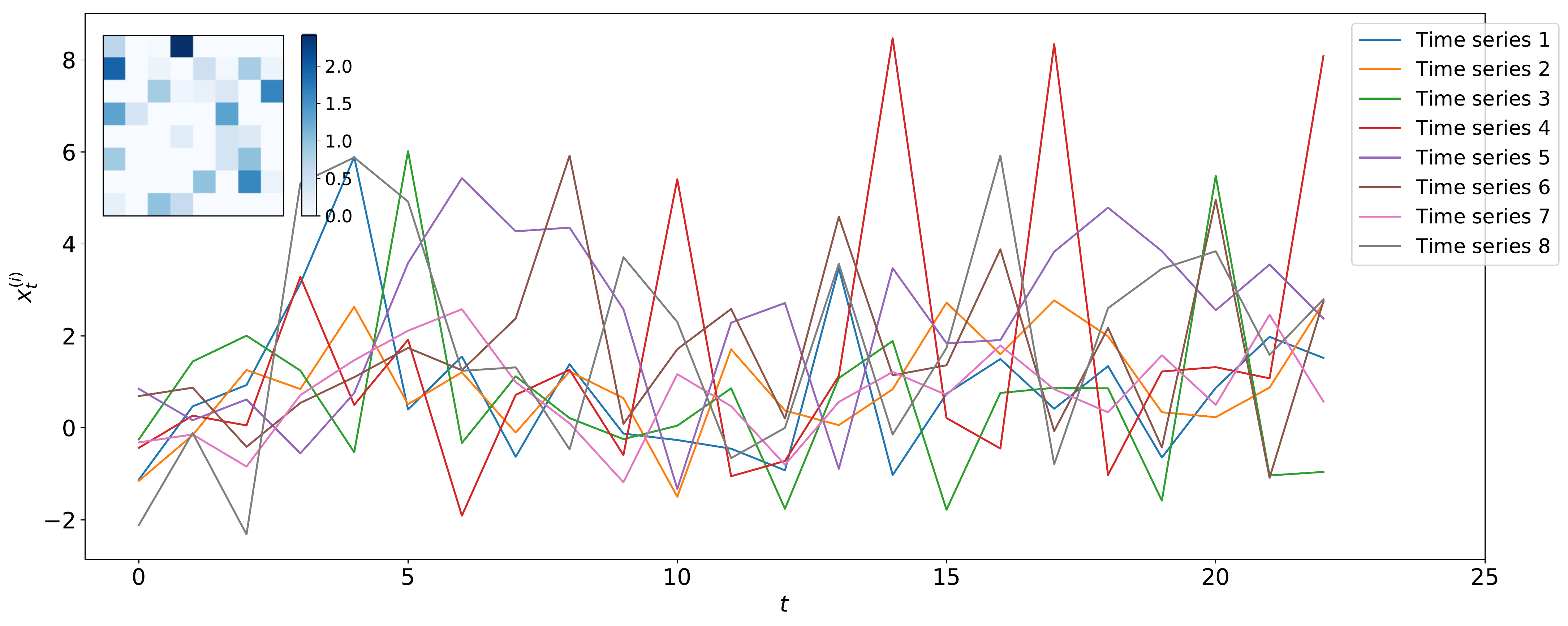}
\caption{}
\end{subfigure}
\begin{subfigure}{1\columnwidth}
\includegraphics[scale=0.33]{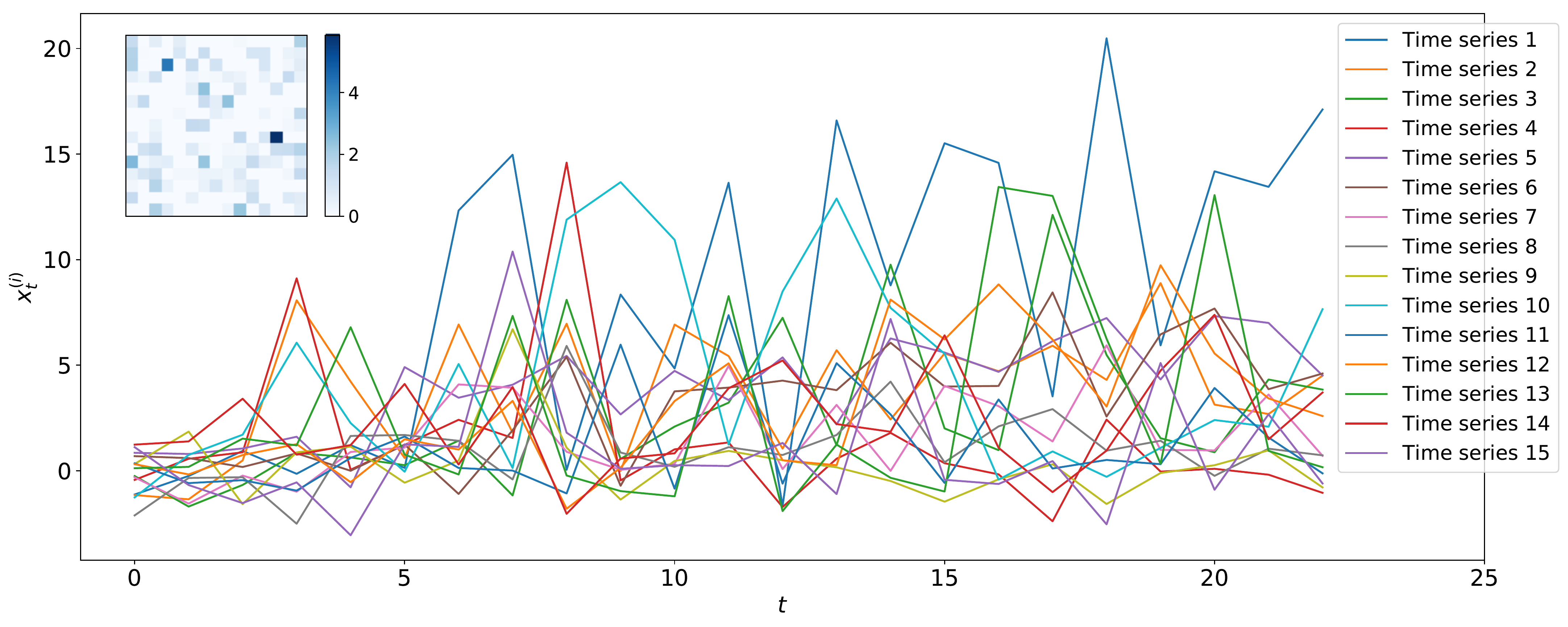}
\caption{}
\end{subfigure}
\begin{subfigure}{1\columnwidth}
\includegraphics[scale=0.33]{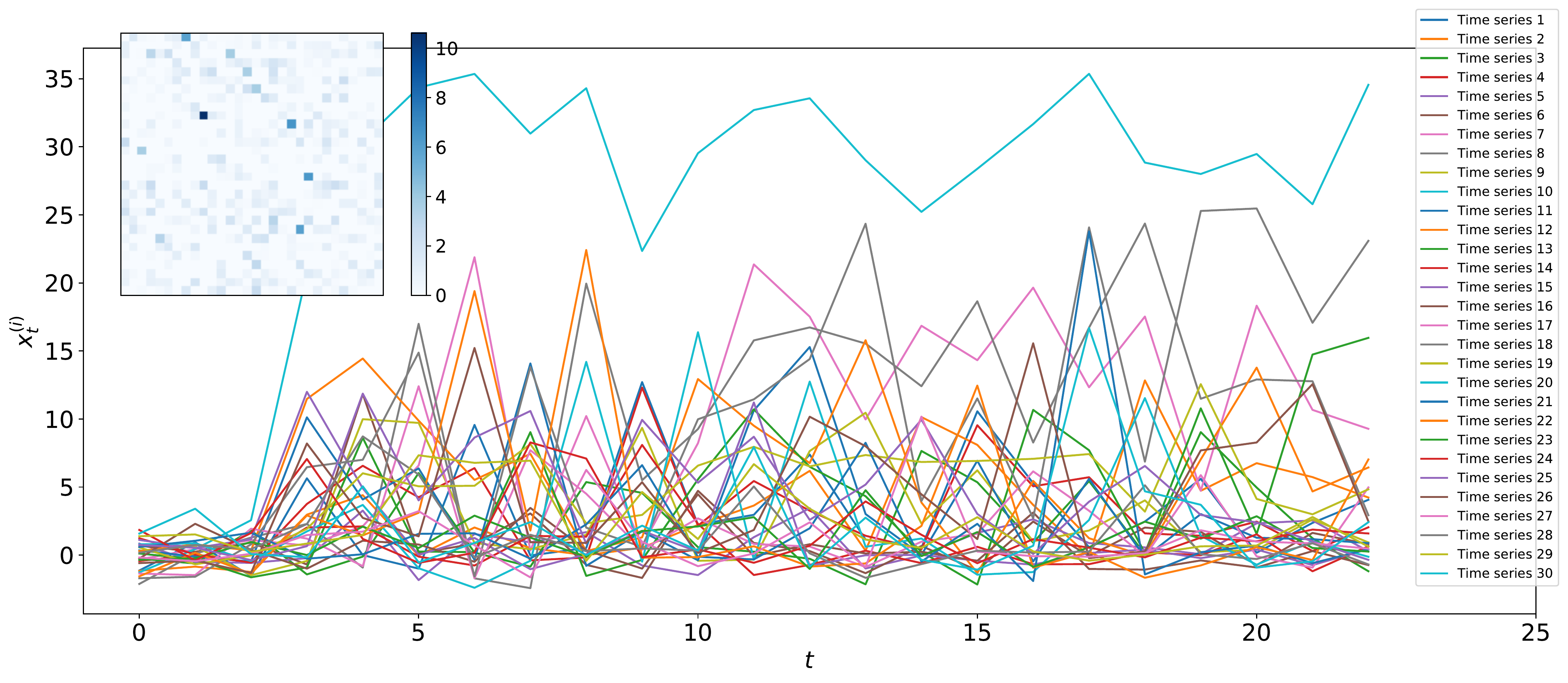}
\caption{}
\end{subfigure}
\caption{Example snapshots of the synthetic time series with (a) $N=8$, (b) $N=15$, and (c) $N=30$. The inset is the hidden underlying $|A_{ji}|$ matrix, whose $(j,i)$ element denotes the causal strength from time series $j$ to $i$. We see that the causal strength varies in orders, making it very difficult to identify each edge correctly.}
\label{fig:synthetic_example_figure}
\end{suppfigure}

\begin{supptable*}[t]
\caption{Mean and standard deviation of AUC-ROC (\%) vs. $N$, over 10 random sampling of datasets. Bold font marks the top method for each $N$.}
\centering
\begin{tabular}{p{2.75cm}p{1cm}p{1cm}p{1cm}p{1cm}p{1cm}p{1cm}p{1cm}p{1cm}}
\toprule
N &      3  &    4  &    5  &    8  &    10 &    15 &    20 &    30 \\
method             &         &       &       &       &       &       &       &       \\
\midrule
\textbf{MPIR (ours)}        &    95.3{\tiny$\pm$10.0} &  97.6{\tiny$\pm$4.1} &  \textbf{97.3}{\tiny$\pm$3.6} &  \textbf{96.0}{\tiny$\pm$2.4} &  \textbf{94.2}{\tiny$\pm$3.8} &  \textbf{91.0}{\tiny$\pm$3.5} &  \textbf{85.5}{\tiny$\pm$2.4} &  \textbf{76.8}{\tiny$\pm$3.5} \\
Mutual Information &    84.1{\tiny$\pm$18.9} &  90.0{\tiny$\pm$7.6} &  89.0{\tiny$\pm$1.8} &  87.2{\tiny$\pm$3.8} &  81.3{\tiny$\pm$5.3} &  77.5{\tiny$\pm$3.9} &  74.6{\tiny$\pm$3.0} &  72.0{\tiny$\pm$2.0} \\
Transfer Entropy   &    88.3{\tiny$\pm$14.6} &  95.6{\tiny$\pm$5.7} &  89.9{\tiny$\pm$8.7} &  84.4{\tiny$\pm$7.6} &  80.8{\tiny$\pm$5.1} &  69.6{\tiny$\pm$2.5} &  64.7{\tiny$\pm$2.5} &   59.2{\tiny$\pm$1.9} \\
Linear Granger     &    \textbf{98.8}{\tiny$\pm$4.0} &  96.2{\tiny$\pm$5.5} &  91.7{\tiny$\pm$8.9} &  84.1{\tiny$\pm$9.0} &  82.7{\tiny$\pm$7.2} &  73.6{\tiny$\pm$6.9} &  69.9{\tiny$\pm$4.1} &  60.0{\tiny$\pm$2.6} \\
Kernel Granger     &    98.1{\tiny$\pm$5.9} &  \textbf{98.0}{\tiny$\pm$4.4} &  95.4{\tiny$\pm$3.9} &  91.2{\tiny$\pm$2.6} &  89.5{\tiny$\pm$3.3} &  82.4{\tiny$\pm$2.2} &  76.2{\tiny$\pm$2.2} &  68.1{\tiny$\pm$1.3} \\
Elastic Net        &    97.5{\tiny$\pm$7.9} &  97.4{\tiny$\pm$4.5} &  95.3{\tiny$\pm$4.3} &  90.4{\tiny$\pm$5.1} &  87.7{\tiny$\pm$4.1} &  81.8{\tiny$\pm$3.1} &  77.8{\tiny$\pm$3.0} &  72.7{\tiny$\pm$1.4} \\
Causal Influence   &    62.9{\tiny$\pm$28.3} &  58.3{\tiny$\pm$13.8} &  60.4{\tiny$\pm$11.7} &  47.4{\tiny$\pm$7.5} &  50.7{\tiny$\pm$5.6} &  55.3{\tiny$\pm$3.3} &  51.0{\tiny$\pm$3.2} &   50.3{\tiny$\pm$1.6} \\
Gaussian random    &    49.9{\tiny$\pm$0.3} &  50.0{\tiny$\pm$0.1} &  50.0{\tiny$\pm$0.1} &  50.0{\tiny$\pm$0.0} &  50.0{\tiny$\pm$0.1} &  50.0{\tiny$\pm$0.0} &  50.0{\tiny$\pm$0.0} &  50.0{\tiny$\pm$0.0} \\
\bottomrule
\end{tabular}
\label{table:synthetic_larger_N_AUC_ROC}
\end{supptable*}

\section{Additional experiment: testing with model capacity variations}
\label{app:capacity}
Since in practice, we do not know the underlying causal structure \textit{a priori}, it presents a greater challenge to select the model capacity for $f_\theta$, as compared with supervised learning method where we can do cross-validation. To see how the capacity of the function approximator $f_\theta$ influences our method, we vary the number of layers and the number of neurons in each layer at $N=10$, using the same 10 datasets as in Section \ref{sec:synthetic}. Table S\ref{table:synthetic_capacity} summarizes the result. We see that our method's performance here is hardly influenced by the model capacity, with only a slight degradation at very low capacity. This shows that our method is quite tolerant and stable with model capacity variations. 

\begin{supptable}[t]
\caption{Average and standard deviation of AUC-PR and AUC-ROC for different network structures for $N=10$ with our method. Here for example, (8, 8, 8) means that the $f_\theta$ has 3 hidden layers, each with 8 neurons.}
    \centering
    \begin{tabular}{lrr}
\toprule
                           &  AUC-PR (\%) &  AUC-ROC (\%) \\
Neurons in hidden layers &         &          \\
\midrule
(8) &  90.0{\small$\pm$4.9} &   91.5{\small$\pm$4.3} \\
  (8, 8) &  93.4{\small$\pm$3.6} &   94.1{\small$\pm$3.7} \\
  (8, 8, 8) &  93.6{\small$\pm$3.6} &   94.4{\small$\pm$3.6} \\
   (8, 8, 8, 8) &  93.8{\small$\pm$4.1} &   94.2{\small$\pm$4.3} \\
 (16, 16) &  94.3{\small$\pm$3.3} &   94.4{\small$\pm$3.5} \\
 (16, 16, 16) &  94.6{\small$\pm$3.0} &   95.1{\small$\pm$2.6} \\
(16, 16, 16, 16) &  92.8{\small$\pm$4.4} &   94.0{\small$\pm$3.2} \\
\bottomrule
\end{tabular}
\label{table:synthetic_capacity}
\end{supptable}

\section{Details for the video game dataset}
\label{app:breakout}

Here, we implement a custom Atari Breakout game in the OpenAI Gym \cite{1606.01540} environment, mimicking the original game\footnote{A game playing video can be seen at \href{https://goo.gl/XGzppc}{https://goo.gl/XGzppc}.}, where we can access the state of the ball, paddle and bricks, etc. This representation is also used in the OO-MDP \cite{diuk2008object} paradigm for a more efficient representation of the environment state. We use the DQN algorithm, the same CNN architecture as in \cite{mnih2015human} to train an RL agent. Then we let it play the game for $\sim$45000 steps, obtaining a dataset with time-length of 45000 steps (if the agent dies, we restart the game) and 6 time series: action, paddle's $x$ position, ball's $x$ position, ball's $y$ position, number of bricks and reward. We then feed the time series (each time series normalized to mean of 0 and variance of 1) to our method, the same procedure as performed in the synthetic experiment, to let it produce an inferred matrix $W_{ji}$, which is shown in Fig. 1 in main text. All the datasets used in this paper and code will be open-sourced upon publication of the paper.

\section{Implementation details for experiment with heart-rate vs. breath-rate}
\label{app:real_dataset}

For the two real-world datasets, we obtain the data with the same procedure as in \cite{ancona2004radial} (See Fig.S\ref{fig:apnea_figure} for their plots). Then the data (each time series normalized to mean of 0 and variance of 1) are fed into our algorithm to infer the causal strength $W_{ji}$. For each $K=1,2,...20$, the experiments are run for 50 times with seed from 0 to 49, and Fig. \ref{fig:apnea} in the main text is obtained by averaging over the inferred $W$ matrix.

\begin{suppfigure}[t]
\begin{center}
\centerline{\includegraphics[width=1.0\columnwidth]{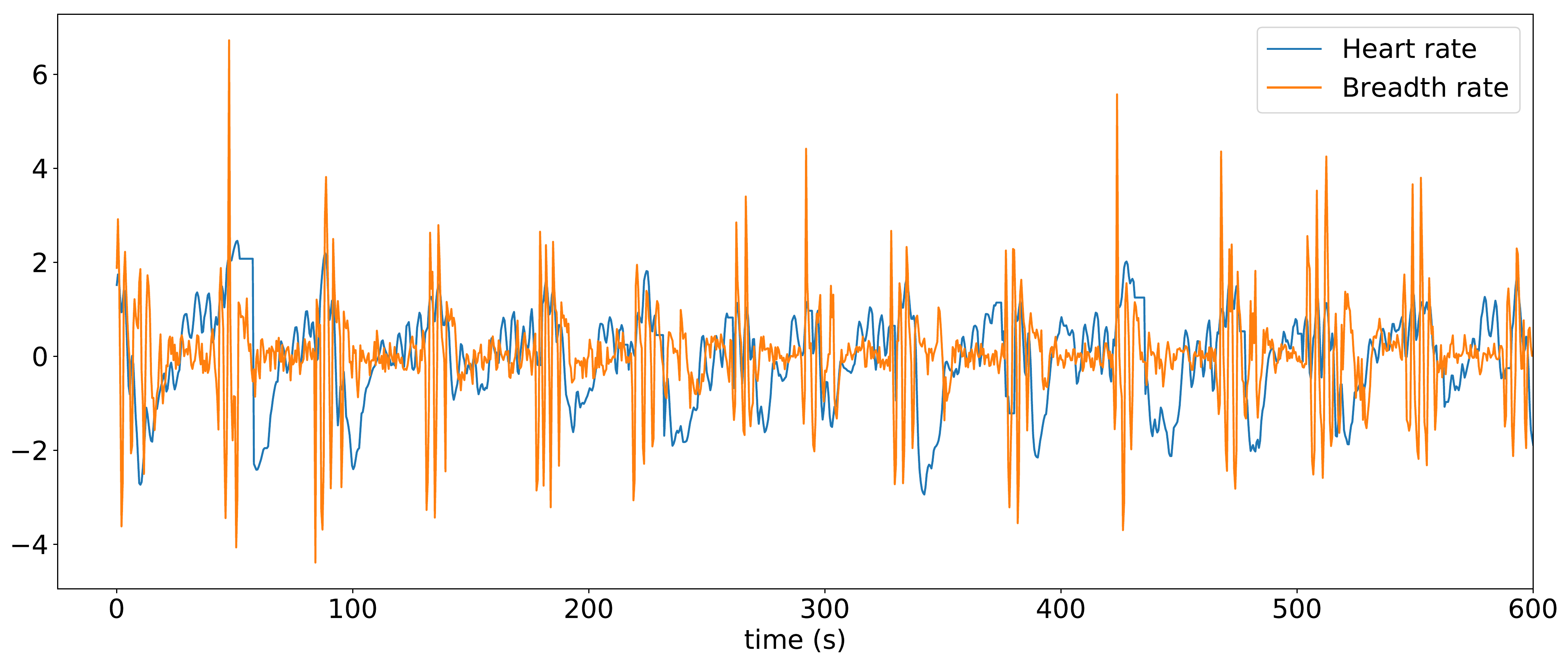}}
\caption{Time series of the heart rate and breath rate of a patient suffering sleep apnea. The data is normalized to have 0 mean and standard deviation of 1. Sample rate is 2Hz.}
\label{fig:apnea_figure}
\end{center}
\end{suppfigure}

\section{Additional experiment: rat EEG dataset}
\label{app:rat_EEG_experiment}

As another real-world example, we apply our algorithm to estimate the directional relations of the EEG signals between the right and left cortical intracranial electrodes \cite{ratEEG}, before and after lesion (see Fig. S\ref{fig:ratEEG_before} and S\ref{fig:ratEEG_after_figure} for the signals), also studied in \cite{ancona2004radial,quiroga2002performance,marinazzo2008kernel}. Figure S\ref{fig:ratEEG_W} (left) shows the inferred predictive strength $W_{ji}$ for the EEG signals of a normal rat. We see that there is only a slight asymmetry, with the right channel having a slightly stronger influence on the left channel than the reverse direction. Fig. S\ref{fig:ratEEG_W} (right) shows $W_{ji}$ for the EEG signals with unilateral lesion in the rostral pole of the reticular thalamic nucleus. We see that there is stronger predictive strength from the left to the right channels. Compared with the result of previous works \cite{ancona2004radial,marinazzo2008kernel} as also shown in Fig. S\ref{fig:ratEEG_compare}, we see that all methods correctly infer the directional relations before and after brain lesion. In addition, our method shows only a slight decay of predictive strength with increasing history length, in contrast to the much more rapid decay of causality index in \cite{ancona2004radial}, again demonstrating our method's insensitivity against history length, due to its flexibility in extracting the right amount of information in order to predict the future. This experiment and the breadth rate vs. heart rate experiment in Section \ref{sec:heart_rate} demonstrate our method's capability in inferring the directional relations from noisy, real-world data.

\begin{suppfigure}[t]
\begin{center}
\centerline{\includegraphics[width=1\columnwidth]{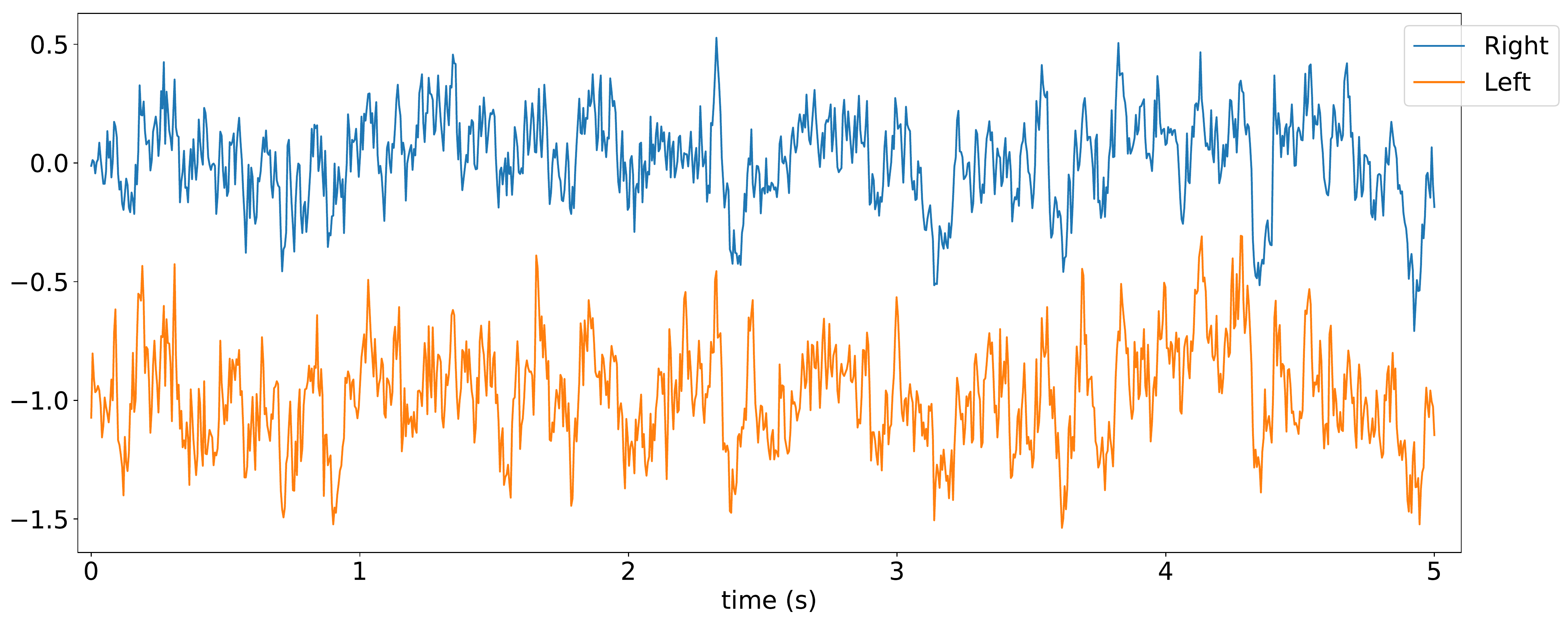}}
\caption{Time series of a normal rat EEG signals from right and left cortical intracranial electrodes. The data is normalized to have 0 mean and standard deviation of 1, and the left signal is plotted with offset for better visualization. Sample rate is 200Hz.}
\label{fig:ratEEG_before}
\end{center}
\vskip -0.3in
\end{suppfigure}

\begin{suppfigure}[t]
\begin{center}
\centerline{\includegraphics[width=1\columnwidth]{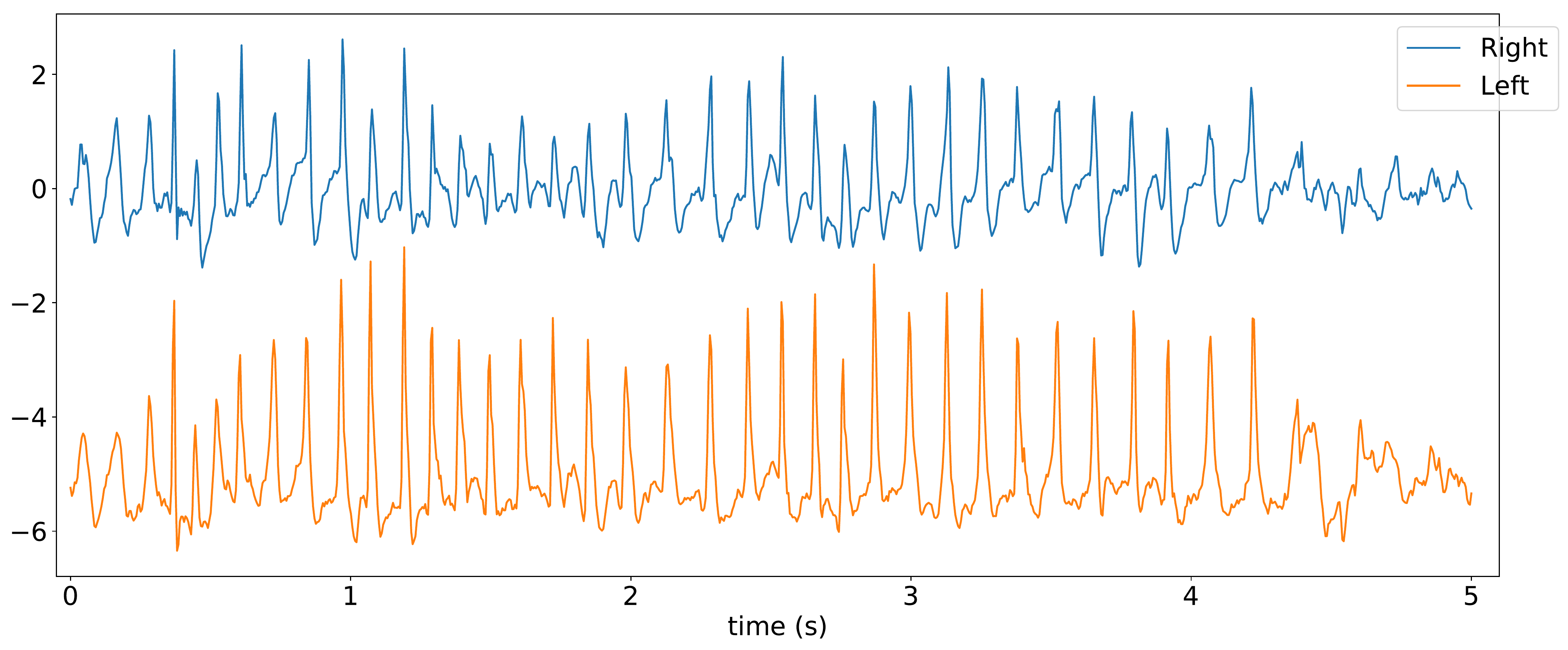}}
\caption{Time series of a rat EEG signals from right and left cortical intracranial electrodes, after lesion. The data is normalized to have 0 mean and standard deviation of 1, and the left signal is plotted with offset for better visualization. Sample rate is 200Hz.}
\label{fig:ratEEG_after_figure}
\end{center}
\vskip -0.3in
\end{suppfigure}

\begin{suppfigure}[t]
\begin{center}
\centerline{\includegraphics[width=0.8\columnwidth]{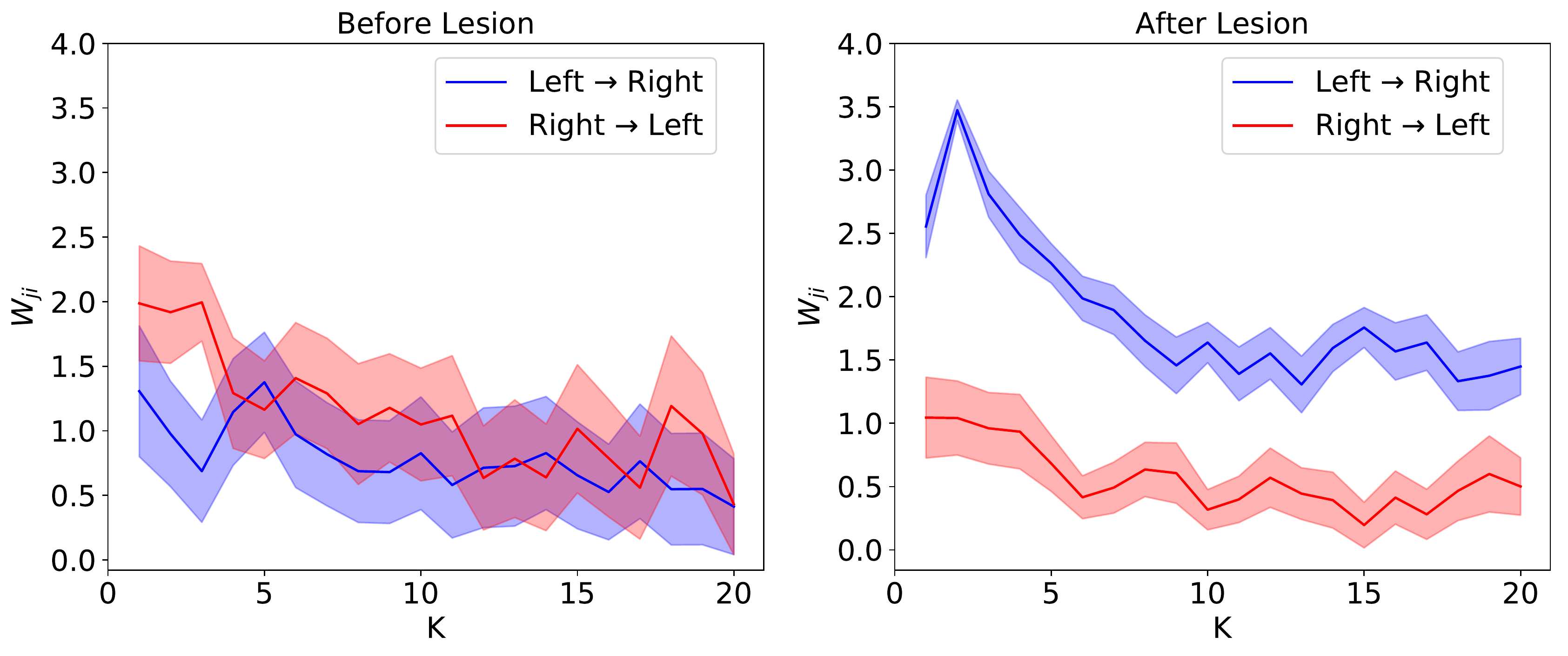}}
\caption{Predictive strength inferred by our method with the EEG datasets, for different maximum time horizon $K$, averaged over 50 initializations of $f_\theta$, for a normal rat (left) and after brain lesion (right).}
\label{fig:ratEEG_W}
\end{center}
\vskip -0.3in
\end{suppfigure}

\begin{suppfigure}
\centering
\begin{subfigure}{.75\linewidth}
\includegraphics[scale=0.55]{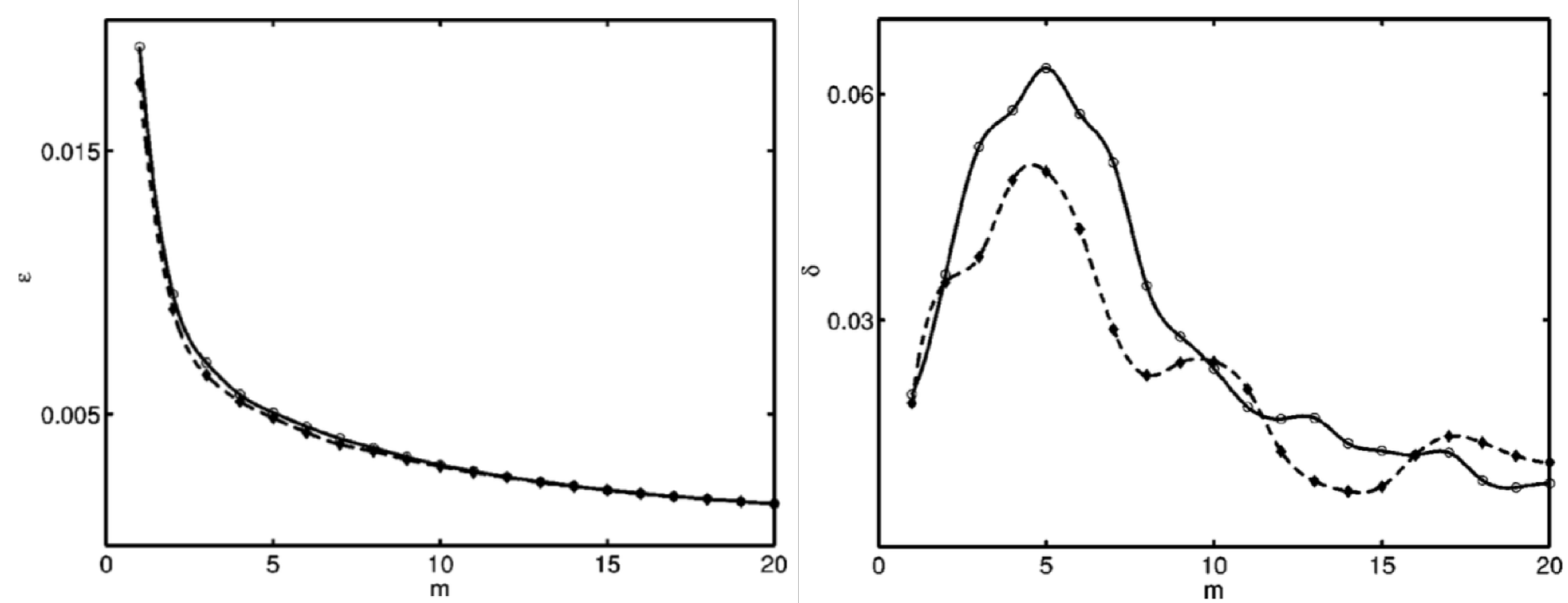}
\caption{}
\end{subfigure}
\begin{subfigure}{.6\linewidth}
\includegraphics[scale=0.55]{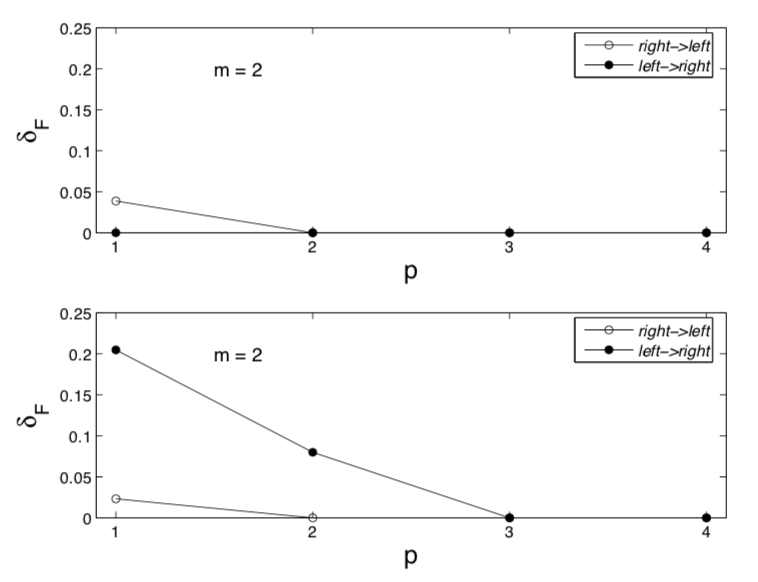}
\caption{}
\end{subfigure}
\caption{Causal indices for the rat EEG dataset with previous methods. (a) By \cite{ancona2004radial}. Left: the variance for the left EEG (open circles) and right EEG (diamonds) vs. time lag $m$ before brain lesion. Right: the causality index after brain lesion. (b) By \cite{marinazzo2008kernel2}. The filtered causality index vs. varying $p$, the order of the inhomogeneous polynomial kernel, before (upper) and after (lower) brain lesion.}
\label{fig:ratEEG_compare}%
\end{suppfigure}

\end{document}